\documentclass[sigconf, nonacm]{acmart}
\usepackage{multirow} 
\AtBeginDocument{%
  }

\setcopyright{acmlicensed}
\copyrightyear{2018}
\acmYear{2018}
\acmDOI{XXXXXXX.XXXXXXX}
\acmConference[Conference acronym 'XX]{Make sure to enter the correct
  conference title from your rights confirmation email}{June 03--05,
  2018}{Woodstock, NY}
\acmISBN{978-1-4503-XXXX-X/2018/06}

\begin{document}

\title{EmergentBridge: Improving Zero-Shot Cross-Modal Transfer in Unified Multimodal Embedding Models}

\author{Jincheng Xie}
\affiliation{%
  \institution{Tsinghua University}
  \city{Beijing}
  \country{China}
}
\email{xiejc22@mails.tsinghua.edu.cn}
\authornote{Both authors contributed equally to this research.}
\author{Xingchen Xiao}
\authornotemark[1]
\affiliation{%
  \institution{School of Computer Science and Technology, Beijing Institute of Technology}
  \city{Beijing}
  \country{China}
}
\email{xcxiao@bit.edu.cn}

\author{Runheng Liu}
\affiliation{%
  \institution{School of Computer Science and Technology, Beijing Institute of Technology}
  \city{Beijing}
  \country{China}
}
\email{rhliu@bit.edu.cn}

\author{Zhongyi Huang}
\affiliation{%
  \institution{Tsinghua University}
  \city{Beijing}
  \country{China}
}
\email{zhongyih@tsinghua.edu}

\author{Yu Zheng}
\affiliation{%
  \institution{JD iCity, JD Technology, JD Intelligent Cities Research}
  \city{Beijing}
  \country{China}}
\email{msyuzheng@outlook.com}

\author{Heyan Huang}
\affiliation{%
 \institution{School of Computer Science and Technology, Beijing Institute of Technology}
 \city{Beijing}
 \country{China}}
\email{hhy63@bit.edu.cn}
\authornote{Corresponding author}

\renewcommand{\shortauthors}{Trovato et al.}

\begin{abstract}
Unified multimodal embedding spaces underpin practical applications such as cross-modal retrieval and zero-shot recognition.
In many real deployments, however, supervision is available only for a small subset of modality pairs (e.g., image--text), leaving \emph{unpaired} modality pairs (e.g., audio$\leftrightarrow$depth, infrared$\leftrightarrow$audio) weakly connected and thus performing poorly on zero-shot transfer. 
Addressing this sparse-pairing regime is therefore essential for scaling unified embedding systems to new tasks without curating exhaustive pairwise data.
We propose \textbf{EmergentBridge}, an embedding-level bridging framework that improves performance on these unpaired pairs \emph{without requiring exhaustive pairwise supervision}.
Our key observation is that naively aligning a new modality to a synthesized proxy embedding can introduce \emph{gradient interference}, degrading the anchor-alignment structure that existing retrieval/classification relies on.
EmergentBridge addresses this by (i) learning a mapping that produces a \emph{noisy bridge anchor} (a proxy embedding of an already-aligned modality) from an anchor embedding, and (ii) enforcing proxy alignment only in the subspace orthogonal to the anchor-alignment direction, preserving anchor alignment while strengthening non-anchor connectivity.
Across nine datasets spanning multiple modalities, EmergentBridge consistently outperforms prior binding baselines on zero-shot classification and retrieval, demonstrating strong emergent alignment.
\end{abstract}

\ccsdesc[300]{Information systems~Image search}
\ccsdesc[300]{Information systems~Speech / audio search}

\keywords{Multimodal Embedding, Unified Representation Space, Cross-modal Retrieval}

\received{20 February 2007}
\received[revised]{12 March 2009}
\received[accepted]{5 June 2009}

\maketitle

\begin{figure}[htbp] 
    \centering
    \includegraphics[width=0.5\textwidth]{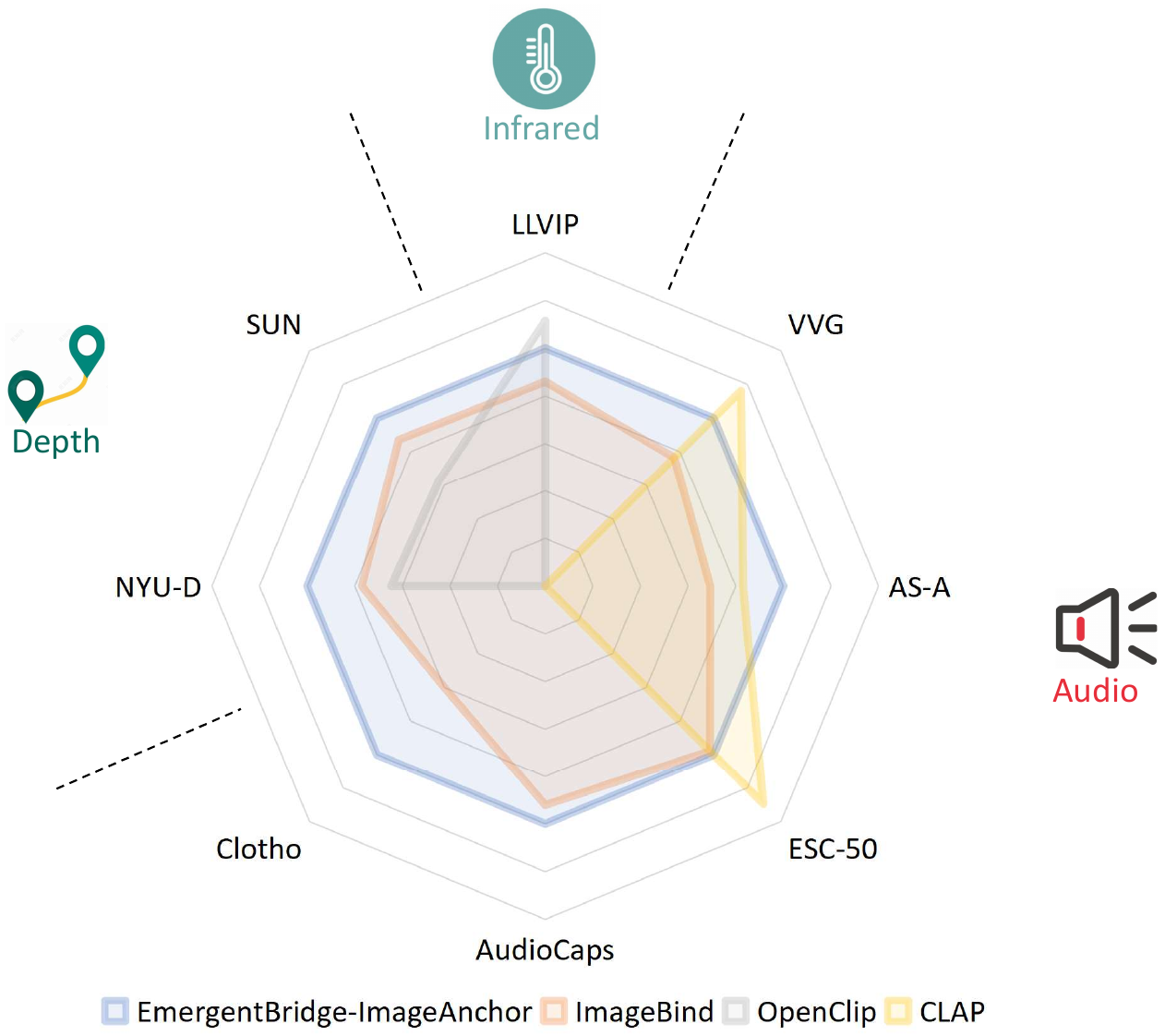} 
    \vspace{-10mm}
    \caption{\textbf{Zero-Shot Language-Related Task Performance:} EmergentBrideg with image anchor modality is demonstrated as a powerful emergent ability for indirect alignment modalities.}  
    \label{fig_f}
    \vspace{-3mm}
\end{figure}
\section{Introduction}

In real-world multimodal systems, models are increasingly expected to reason across heterogeneous sensory inputs—such as vision, depth, infrared, and audio—despite the fact that many modality pairs are never directly observed together during training.
To address this challenge, recent omni-capable multimodal foundation models adopt a unified embedding space as a common interface, where modality-specific encoders map diverse inputs into a shared representation space \citep{Wang2023ONEPEACEEO,girdhar2023imagebindembeddingspacebind,guzhov2021audioclipextendingclipimage,Wu2021Wav2CLIPLR,Wu2022LargeScaleCL,OpenShape}.

\begin{figure*}
\vspace{-3mm}
    \centering
    \includegraphics[width=0.95\linewidth]{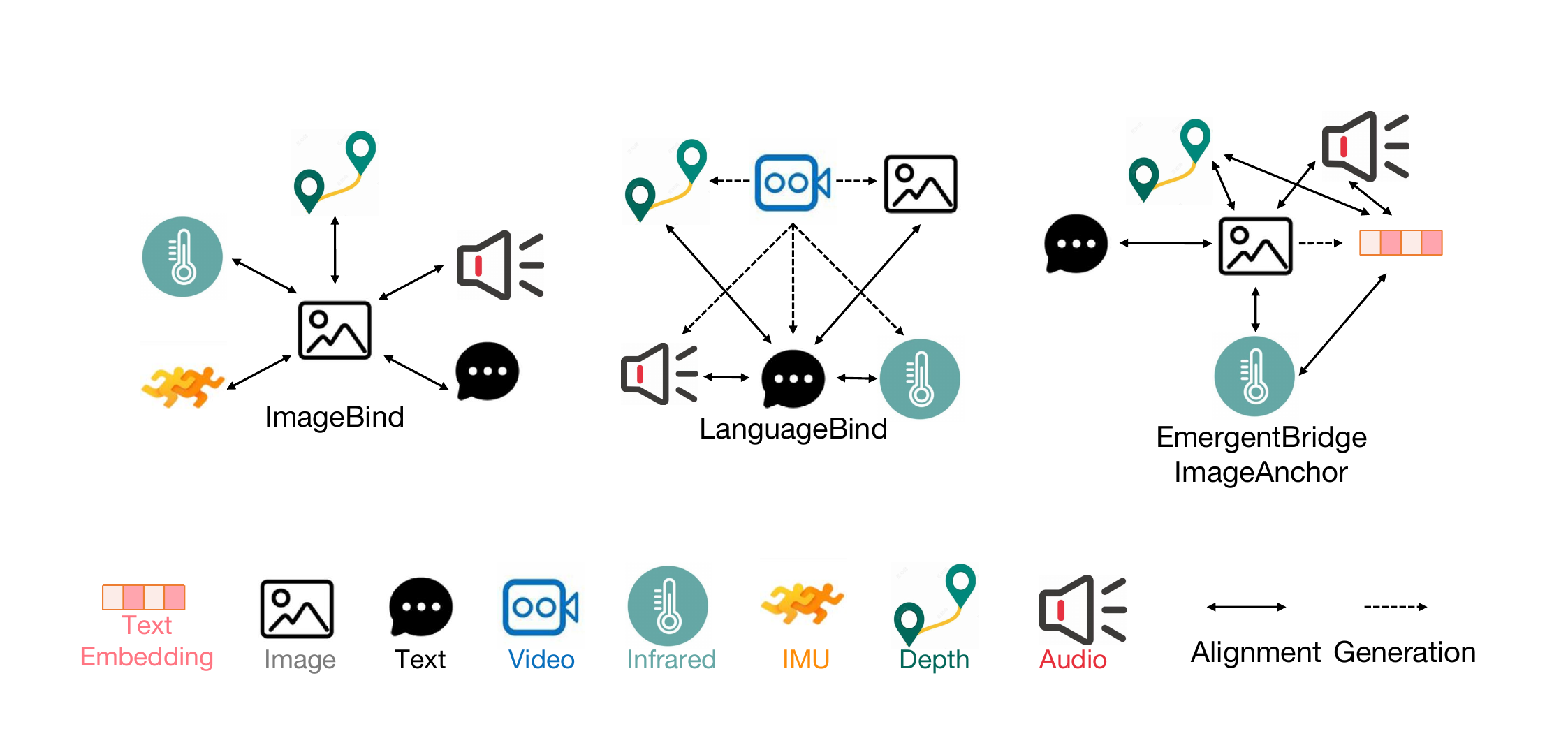}
    \vspace{-0.5cm}
    \caption{\textbf{EmergentBridge vs. ImageBind and LanguageBind:}
The left image illustrates ImageBind's approach, where modalities are indirectly aligned through direct alignment with the image modality. The center image depicts LanguageBind's strategy of augmenting the model's zero-shot capabilities by generating data. In contrast, the right image demonstrates EmergentBridge, which uses the image as the core modality to enhance emergent zero-shot capabilities by generating embeddings rather than additional data.
    }
    \label{fig:compare}
\vspace{-5mm}
\end{figure*}

However, a unified embedding space alone is insufficient: because supervision typically covers only a small subset of modality pairs, modalities that are never directly paired during training can remain weakly and inconsistently aligned, leading to brittle or even failed zero-shot transfer at inference time  \citep{zhu2024languagebindextendingvideolanguagepretraining}.
A key reason is that current multimodal embedding systems are typically optimized around a small set of \emph{anchor-supervised} modality pairs (e.g., image--text).
While individual modalities may each align well with the anchor, this does not ensure reliable alignment between \emph{unpaired} modality pairs (e.g., audio$\leftrightarrow$depth). 
As a result, a model can perform strongly on anchor-supervised pairs yet still perform poorly in heterogeneous, sparsely supervised settings.

This behavior indicates a missing capability in current multimodal embedding spaces: the ability to meaningfully relate two modalities that were never seen together during training.
We refer to this capability as \emph{emergent alignment}, which is essential for unified representations to generalize beyond anchor-supervised settings.
Without emergent alignment, strong alignment to a shared anchor does not translate into reliable cross-modal reasoning between unseen modality pairs.
We make emergent alignment observable by measuring zero-shot transfer on unpaired modality pairs via cross-modal retrieval and zero-shot classification.

To avoid exhaustive pairwise supervision, most existing multimodal systems adopt anchor-based binding, where each modality is independently aligned to a single anchor modality \citep{Gao2024MEDBindUL,FreeBind,zhu2024languagebindextendingvideolanguagepretraining,geobind,omnibind,unibind}.
While effective for scalable pretraining and anchor-supervised tasks, this paradigm leaves relationships between non-anchor modalities implicitly defined and weakly constrained.

As a result, transfer between unpaired modality pairs remains fragile, for example in ImageBind \citep{girdhar2023imagebindembeddingspacebind}, where strong alignment to images does not guarantee reliable cross-modal reasoning between non-image modalities.
More broadly, \textbf{existing binding approaches lack a mechanism to explicitly strengthen emergent alignment without degrading well-supervised anchor alignments, revealing a fundamental tension between scalability and cross-modality consistency.}

To resolve this tension, we introduce \textbf{EmergentBridge}, a lightweight framework that improves cross-modal connectivity without disrupting existing anchor-based alignments.
EmergentBridge can be applied on top of pretrained multimodal embedding models, avoiding costly re-pretraining or the need for full pairwise supervision.

EmergentBridge adds an intermediate embedding-level bridge between modalities.
Given an anchor embedding, the framework synthesizes a proxy embedding for an already-aligned modality, which serves as an additional alignment target for a new modality.
By jointly aligning to both the anchor and the synthesized proxy, the model strengthens relationships between unpaired modalities while maintaining the scalability advantages of anchor-based training.

Across nine datasets spanning audio, depth, and infrared modalities, \textbf{EmergentBridge} achieves substantial improvements in zero-shot transfer on unpaired modality pairs, while preserving strong performance on anchor-supervised tasks. These results demonstrate that emergent alignment can be explicitly strengthened without sacrificing the core utility of unified embedding spaces.

Our primary contributions are summarized as follows:
\begin{itemize}
\item We identify \emph{emergent alignment} as a critical and underexplored capability of unified multimodal embedding spaces, and characterize it through zero-shot transfer on \emph{unpaired} modality pairs under limited pairwise supervision.
\item We propose \textbf{EmergentBridge}, a lightweight embedding-level framework that explicitly strengthens connectivity between non-anchor modalities by synthesizing proxy embeddings from anchor-supervised pairs, without retraining large backbones or requiring exhaustive pairwise data.
\item To mitigate interference from noisy proxy signals, we introduce an \emph{orthogonal-subspace alignment} regularizer, and provide theoretical conditions under which anchor-modality alignment is preserved.
\item Extensive experiments across nine datasets spanning audio, depth, and infrared demonstrate substantial improvements in zero-shot transfer on unpaired modality pairs—averaging 24.7\% in classification and 49\% in retrieval—while largely preserving performance on well-supervised anchor pairs.
\end{itemize}
\section{Related Work}

\paragraph{Multimodal Learning}
CLIP \citep{radford2021learningtransferablevisualmodels} is a seminal multimodal learning method that aligns images and text to build cross-modal representations. Numerous approaches, including CLIP4Clip \citep{clip4clip} and Clip2Video \citep{CLIP2Video}, extend CLIP to extract semantic visual representations. 
Recent efforts have extensively explored multimodal alignment through pretraining \citep{yin2023survey,xu2023multimodal,wu2023multimodal}. Beyond vision and language modalities, AudioCLIP \citep{guzhov2021audioclipextendingclipimage} incorporates audio into the CLIP framework, enabling zero-shot audio classification. ImageBind \citep{girdhar2023imagebindembeddingspacebind} advances multimodal alignment by aligning all modalities with the vision modality. ImageBind-LLM \citep{Han2023ImageBindLLMMI} leverages the joint embedding space of pre-trained ImageBind to efficiently fine-tune LLaMA. NeuroBind \citep{Yang2024NeuroBindTU} learns a generalized representation that unifies various types of brain signal based on the pretrained image embedding space. UniBind \citep{unibind} adaptively constructs class-wise LLM-augmented embedding centers to achieve a unified and balanced representation space. To improve performance on language-related tasks, MEDBind \citep{Gao2024MEDBindUL} and LanguageBind \citep{zhu2024languagebindextendingvideolanguagepretraining} use text data as the main modality to align other modalities. However, these methods have not investigated approaches to enhance emergent capabilities for previously untrained modality pairs.

\paragraph{Contrastive Learning}
Contrastive learning has achieved significant success in learning representations of multimodal data pairs \citep{logeswaran2018efficient, he2020momentum}. The primary objective of these approaches is to maximize the mutual information between the paired views \citep{contrastive, bachman2019learning, Tamkin2020ViewmakerNL}. Loss functions such as NCE \citep{nce}, InfoNCE \citep{Oord2018RepresentationLW}, and MIL-NCE \citep{milnce} have been introduced to facilitate contrastive learning. However, these loss functions focus primarily on aligning two modalities. To extend the framework to multiple modalities, previous studies \citep{guzhov2021audioclipextendingclipimage, alayrac2020self} propose a simple method that sums the loss functions to allow joint learning with various modalities.
According to \citet{understandingInfoNCE}, the loss function of contrastive learning can be split into alignment and uniformity parts. 
The alignment part is responsible for the alignment when using loss function summation for multimodal alignment.
We analyze the direct summation of two loss functions, and then the two alignment parts will interfere in the Appendix \ref{Discussion}.
\section{Preliminaries}

\paragraph{InfoNCE for Modality Alignment.}
Let $\mathcal{C}$ denote an \emph{anchor} modality (e.g., image or text) and $\mathcal{M}_b$ a new modality to be aligned with it.
Given paired embeddings $\{(x_i^b, c_i)\}_{i=1}^N$, the InfoNCE objective for aligning $\mathcal{M}_b$ to $\mathcal{C}$ is
\begin{equation}
    \mathcal{L}^{ \text{infoNCE}}_{\mathcal{M}_b\rightarrow\mathcal{C}} =-\frac{1}{N}\sum_{i=1}^N \log\left(\frac{\exp(\text{sim}(x_i^b,c_i)/\tau)}{\sum_{j=1}^N \exp(\text{sim}(x_i^b,c_j)/\tau)}\right),
\label{eq:infoNCE}
\end{equation}
where $\tau$ is the temperature and $\text{sim}(\cdot,\cdot)$ is a similarity function.
We can rewrite Eq.~\ref{eq:infoNCE} as a sum of (i) a positive-pair alignment term and (ii) a log-partition term over in-batch negatives:
\begin{equation}
\label{eqn:align part}
  \mathcal{L}^{\text{align}}_i = -\text{sim}(x^b_i, c_i) /\tau,
\end{equation}
\begin{equation}
\label{eqn:uniform part}
  \mathcal{L}^{\text{neg}}_i = \log\left(\sum_{j=1}^N \exp(\text{sim}(x^b_i,c_j)/\tau)\right).
\end{equation}
Intuitively, $\mathcal{L}^{\text{align}}_i$ pulls $x^b_i$ toward its matched anchor $c_i$, while $\mathcal{L}^{\text{neg}}_i$ pushes $x^b_i$ away from other anchors, inducing a repulsive effect often associated with representation ``uniformity''.

To characterize the anchor-alignment direction, we compute the gradient of Eq.~\ref{eqn:align part} with respect to $x^b_i$:
\begin{equation}
\label{gradci}
    \bar{c}_i \triangleq -\frac{\partial \mathcal{L}^{\text{align}}_i}{\partial x^b_i}
    = \frac{1}{\tau}\frac{\partial \text{sim}(x_i^b, c_i)}{\partial x_i^b},
\end{equation}
which indicates the (local) direction that increases similarity between $x_i^b$ and $c_i$.
For the common choice of cosine similarity with $\ell_2$-normalized embeddings, $\bar{c}_i$ is aligned with $c_i$ up to a scaling/projection term, providing a concrete geometric notion of the ``anchor-alignment'' direction.

\paragraph{Orthogonal-Subspace Projection.}
EmergentBridge strengthens connectivity between indirectly aligned modalities using a synthesized proxy embedding, while aiming to preserve the anchor alignment induced by InfoNCE.
Given a nonzero direction vector $v\in\mathbb{R}^d$ (instantiated as $v=\bar{c}_i$ in our method), we define the projection onto the orthogonal complement of $\mathrm{span}(v)$:
\begin{equation}
    P^\perp_{v}(x) \overset{\Delta}{=} \left(I-\frac{vv^T}{\|v\|^2+\epsilon}\right)x,
\end{equation}
and optionally normalize it as $\text{normalize}(P^\perp_{v}(x))$.
This operator removes the component of $x$ along $v$.
In EmergentBridge, we apply $P^\perp_v(\cdot)$ to the representation (equivalently, the update direction) used for \emph{proxy alignment}, so that proxy-driven optimization acts in directions that reduce interference with anchor alignment.

\section{Method}
\label{sec:method}

\begin{figure*}[t]
    \centering
    \includegraphics[width=0.95\linewidth]{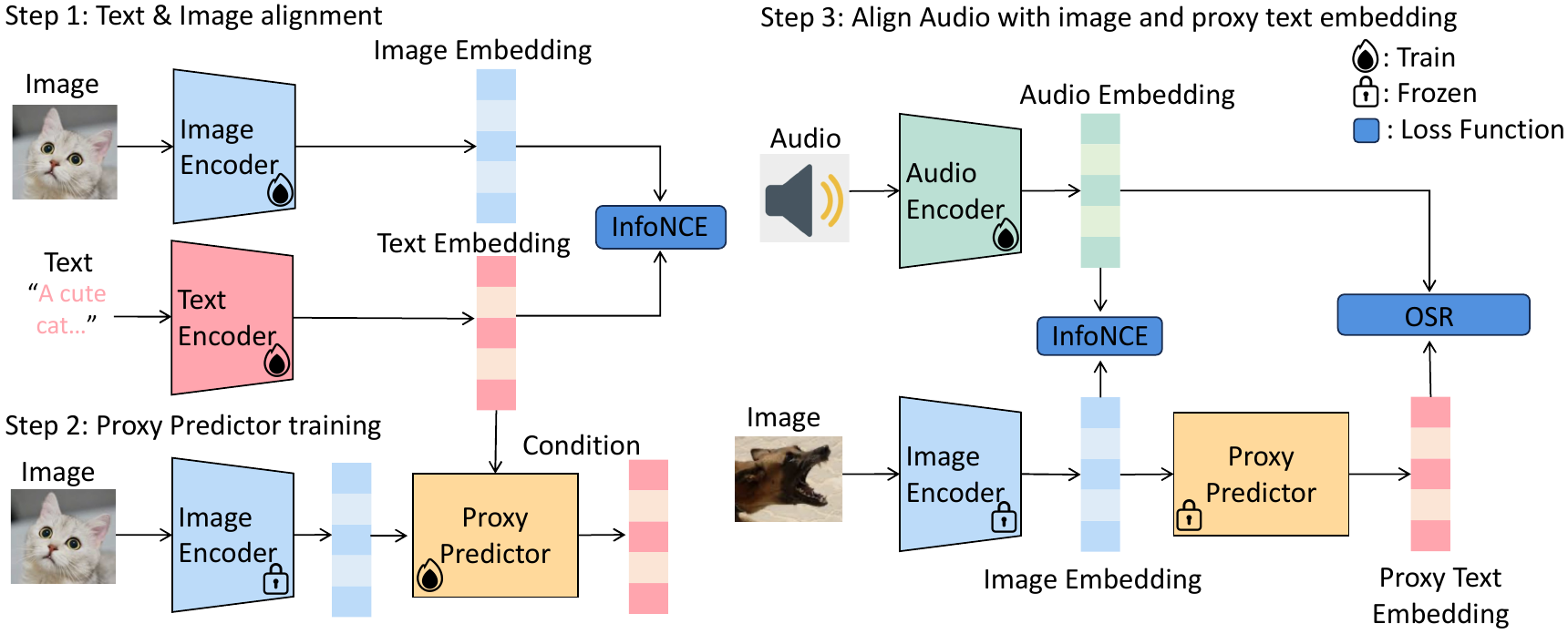}
    \caption{\textbf{An Example Overview of EmergentBridge.} 
    $\mathcal{M}_a,\mathcal{M}_b$ and $\mathcal{C}$ correspond to text, audio, and image, respectively. 
    (Step 1) We align the anchor modality (image) with an already-aligned modality (text) using InfoNCE.
    (Step 2) We train a proxy predictor that maps image embeddings to proxy text embeddings, while keeping the image/text encoders frozen.
    (Step 3) We align audio with both the image embedding and the synthesized text proxy, using InfoNCE together with an orthogonal-subspace regularizer that reduces interference with anchor alignment.
    In practice, when $\mathcal{M}_a$ is already anchor-aligned in a pretrained binding model, Step 1 can be skipped and we only train the proxy predictor and the new modality encoder.}
    \label{fig:pipeline}
\end{figure*}

\textbf{EmergentBridge} is an embedding-level bridging framework that strengthens zero-shot transfer on \emph{unpaired} modality pairs without requiring exhaustive pairwise supervision or re-pretraining large backbones.
The key idea is to introduce an intermediate embedding-level bridge that enables a new modality to align not only with an anchor modality, but also with a synthesized proxy of an already-aligned modality.
As illustrated in Figure~\ref{fig:compare}, this design departs from anchor-only alignment and data-level synthesis by operating directly at the embedding level, explicitly strengthening connectivity between indirectly related modalities while preserving the structure induced by anchor-supervised pairs.

Based on this intuition, the training procedure of EmergentBridge consists of three stages:
\begin{enumerate}
    \item \textbf{Anchor alignment (optional):} align a chosen modality $\mathcal{M}_a$ with the anchor modality $\mathcal{C}$ via contrastive learning.
    \label{step1}
    \item \textbf{Proxy embedding synthesis:} train a proxy predictor that maps anchor embeddings in $\mathcal{C}$ to proxy embeddings of the already-aligned modality $\mathcal{M}_a$.
    \label{step2}
    \item \textbf{Bridge alignment:} align another modality $\mathcal{M}_b$ with both the anchor modality $\mathcal{C}$ and the synthesized proxy embedding of $\mathcal{M}_a$, strengthening cross-modal connectivity without requiring direct supervision between $\mathcal{M}_a$ and $\mathcal{M}_b$.
    \label{step3}
\end{enumerate}

\noindent\textbf{Practical expansion setting.}
When a pretrained binding model already provides an anchor-aligned modality $\mathcal{M}_a$ (e.g., image--text), Step~\ref{step1} can be omitted and we directly reuse the frozen encoders.
When incorporating a new modality $(N{+}1)$, EmergentBridge reuses any previously aligned modality as $\mathcal{M}_a$, treats the new modality as $\mathcal{M}_b$, and repeats Steps~\ref{step2}--\ref{step3}.
Figure~\ref{fig:pipeline} illustrates an example instantiation: text is first aligned with images (anchor), and audio is then aligned with both image embeddings and synthesized text proxies.

\subsection{Anchor Alignment and Proxy Predictor Training}
\label{sec:proxy_predictor}

Following prior binding frameworks, EmergentBridge first establishes (or reuses) anchor alignment between $\mathcal{M}_a$ and $\mathcal{C}$ via contrastive learning.
To bridge other modalities through $\mathcal{M}_a$ without collecting exhaustive pairwise supervision, we train a \emph{proxy predictor} $N_\Theta$ that takes an anchor embedding $c_i \in \mathcal{C}$ as input and predicts the corresponding embedding $x^a_i \in \mathcal{M}_a$:
\begin{equation}
\label{eq:genloss}
\mathcal{L}_{\text{proxy}} = \left\|N_\Theta(c_i) - x^a_i\right\|^2.
\end{equation}
To ensure consistency with encoder outputs, we project the predictor output onto the unit hypersphere,
\begin{equation}
\hat{x}^a_i = \text{normalize}\big(N_\Theta(c_i)\big),
\end{equation}
and use $\hat{x}^a_i$ as a \emph{noisy bridge anchor} (proxy) rather than a ground-truth reconstruction target.
Since encoder embeddings are $\ell_2$-normalized in our setup, minimizing Eq.~\ref{eq:genloss} is consistent (up to constants) with maximizing cosine similarity after normalization.

\subsection{Orthogonal-Subspace Regularizer}
\label{sec:orthoreg}

Given a new modality $\mathcal{M}_b$, our goal is to align it with both the anchor modality $\mathcal{C}$ and the already-aligned modality $\mathcal{M}_a$.
A straightforward approach is to optimize similarity between $x^b_i$ and both $c_i$ and the synthesized proxy embedding $\hat x^a_i$ using InfoNCE.
However, the proxy $\hat x^a_i$ is imperfect; naively enforcing proxy alignment can introduce \emph{gradient interference} and distort the anchor-alignment structure that makes the shared embedding space useful.

To mitigate this, EmergentBridge introduces an \textbf{orthogonal-subspace regularizer} (OSR) that constrains proxy-driven alignment to directions orthogonal to the anchor-alignment direction induced by InfoNCE.
Let $\bar{c}_i$ denote the anchor-alignment direction (Eq.~\ref{gradci}), and define a stop-gradient operator $\mathrm{sg}(\cdot)$ that treats its argument as a constant during backpropagation.
We define the orthogonal-subspace normalization operator
\begin{equation}
\label{eq:tangent_operator}
T_{\bar{c}_i}(x) \overset{\Delta}{=} \text{normalize}\!\left(\left(I-\frac{\mathrm{sg}(\bar{c}_i)\mathrm{sg}(\bar{c}_i)^T}{\|\mathrm{sg}(\bar{c}_i)\|^2+\epsilon}\right)x\right),
\end{equation}
which removes the component of $x$ along the anchor-alignment direction and normalizes it to the unit hypersphere.
Stopping gradients through $\bar{c}_i$ avoids second-order effects; thus the regularizer only constrains the proxy-alignment update to act in directions orthogonal to anchor alignment (Figure~\ref{fig:orthoreg_fig}(a)).

We then apply InfoNCE in this orthogonal subspace to align $\mathcal{M}_b$ to the proxy bridge anchors:
\begin{multline}
\mathcal{L}_{\mathcal{M}_b\rightarrow\mathcal{M}_a}^{\text{osr}}
= -\frac{1}{N}\sum_{i=1}^N \log\left(
\frac{\exp(\text{sim}(T_{\bar{c}_i}(x^b_i), \hat{x}_i^a)/\tau)}
{\sum_{j=1}^N \exp(\text{sim}(T_{\bar{c}_i}(x^b_i), \hat{x}_j^a)/\tau)}\right),
\label{eq:tanloss}
\end{multline}
where we use in-batch negatives $\{\hat x^a_j\}_{j=1}^N$.
Intuitively, this objective encourages $x_i^b$ to move toward $\hat x_i^a$ only through directions that do not interfere with the anchor-alignment direction.

\begin{figure}[t]
    \centering
    \begin{minipage}[t]{0.48\linewidth}
        \centering
        \includegraphics[width=\linewidth]{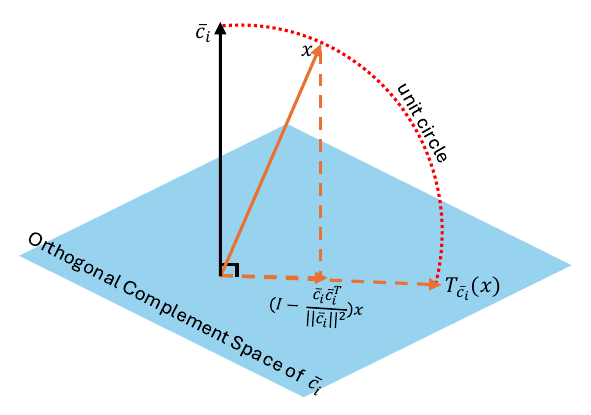}
        \vspace{-1mm}

        {\small \textbf{(a)} Orthogonal-subspace normalization.}
    \end{minipage}
    \hfill
    \begin{minipage}[t]{0.48\linewidth}
        \centering
        \includegraphics[width=\linewidth]{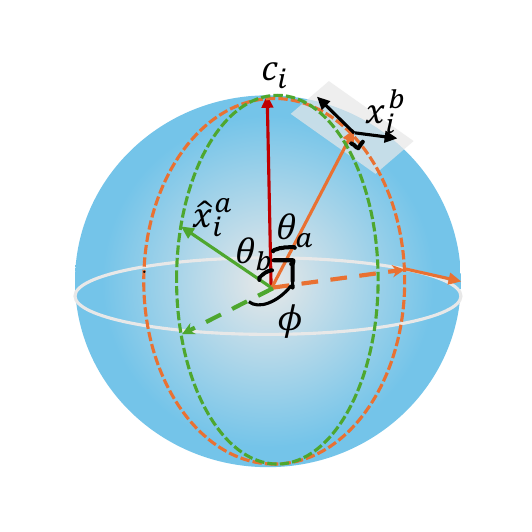}
        \vspace{-1mm}

        {\small \textbf{(b)} Simultaneous alignment.}
    \end{minipage}

    \vspace{-1mm}
    \caption{\textbf{Orthogonal-subspace regularization.}
    (a) $T_{\bar c_i}$ projects $x$ onto the orthogonal complement of the anchor-alignment direction $\bar c_i$ and normalizes it.
    (b) The InfoNCE objective pulls $x_i^b$ toward $c_i$ (black arrow), while the orthogonal-subspace regularizer guides $x_i^b$ toward $\hat x_i^a$ \emph{within} directions orthogonal to $\bar c_i$, reducing interference with anchor alignment.}
    \label{fig:orthoreg_fig}
\end{figure}

Finally, the overall objective for training $\mathcal{M}_b$ combines anchor alignment and orthogonal-subspace proxy alignment:
\begin{equation}
\label{eq:overall_loss}
\mathcal{L} = \mathcal{L}^{\text{infoNCE}} + \lambda \mathcal{L}^{\text{osr}},
\end{equation}
where we use symmetric alignment in both directions,
\begin{equation}
\small
\label{eq:symmetric_losses}
\mathcal{L}^{\text{osr}} = \tfrac{1}{2}\Big(\mathcal{L}^{\text{osr}}_{\mathcal{M}_b\rightarrow\mathcal{M}_a}+\mathcal{L}^{\text{osr}}_{\mathcal{M}_a\rightarrow\mathcal{M}_b}\Big),\quad
\mathcal{L}^{\text{infoNCE}} = \tfrac{1}{2}\Big(\mathcal{L}^{\text{infoNCE}}_{\mathcal{M}_b\rightarrow\mathcal{C}}+\mathcal{L}^{\text{infoNCE}}_{\mathcal{C}\rightarrow\mathcal{M}_b}\Big).
\end{equation}

Figure~\ref{fig:orthoreg_fig}(b) illustrates how the combined objective enables $x^b_i$ to align with both $c_i$ and $\hat{x}^a_i$.
The InfoNCE component draws $x^b_i$ closer to $c_i$, while the orthogonal-subspace regularizer guides $x^b_i$ toward $\hat{x}^a_i$ only through directions orthogonal to the anchor-alignment direction, thereby preserving anchor-modality alignment.
\section{Theoretical Analysis}
\label{sec:theory}

We analyze how the proposed OSR affects \emph{anchor alignment}, i.e., the similarity between $x_i^b$ and $c_i$. 
Our focus is on the positive-pair alignment component of InfoNCE; we treat the negative (log-partition) component as approximately unchanged in a local step and defer its detailed behavior to established analyses of contrastive learning \citep{mindthegap,understandingInfoNCE}. 
This isolates the novel mechanism of EmergentBridge: constraining proxy-driven optimization to an orthogonal subspace to reduce interference with anchor alignment.

\paragraph{Setup.}
Recall the positive-pair alignment term for aligning $\mathcal{M}_b$ to the anchor modality $\mathcal{C}$:
\begin{equation}
\label{eq:theory_align}
\mathcal{L}^{\text{align}}_i(\Theta) = -\text{sim}(x_i^b(\Theta), c_i)/\tau,
\end{equation}
and let $\bar c_i$ denote the anchor-alignment direction in embedding space (Eq.~\ref{gradci}).
Let $\mathcal{L}^{\text{osr}}$ be the OSR proxy-alignment loss (Eq.~\ref{eq:tanloss} in \S\ref{sec:orthoreg}), which aligns $T_{\bar c_i}(x_i^b)$ with the proxy $\hat x_i^a$ in the orthogonal complement of $\bar c_i$. 
In implementation, $\bar c_i$ is treated as a direction indicator and gradients are stopped through it (via $\mathrm{sg}(\cdot)$ in Eq.~\ref{eq:tangent_operator}), so OSR does not introduce second-order effects through $\bar c_i$.

We consider a gradient-based update on parameters $\Theta$:
\begin{equation}
\label{eq:update}
\Delta\Theta = -\delta\Big(\nabla_\Theta \mathcal{L}^{\text{align}} + \lambda \nabla_\Theta \mathcal{L}^{\text{osr}}\Big),
\end{equation}
where $\delta$ is the step size and $\lambda$ weights OSR.

\paragraph{Main result: OSR does not increase the alignment loss to first order under a mild condition.}
The key quantity governing the first-order change of $\mathcal{L}^{\text{align}}$ under the combined update is the inner product between the combined gradient and the alignment gradient. 
We formalize this in the following theorem.

\begin{theorem}[First-order preservation of anchor alignment]
\label{thm:main}
Assume $\bar c_i \neq 0$ and OSR uses stop-gradient through $\bar c_i$ in $T_{\bar c_i}(\cdot)$.
If
\begin{equation}
\label{eq:lambda_condition_general}
\lambda \leq \frac{\|\bar{c}_i\|}{\|\frac{\partial \mathcal{L}^\text{osr}}{\partial T_{\bar{c}_i}(x^b_i)}\|},
\end{equation}
then the combined update direction is non-adversarial to anchor alignment in the sense that
\begin{equation}
\label{eq:nonnegative_inner_product}
\Big(\nabla_\Theta \mathcal{L}^{\text{align}} + \lambda \nabla_\Theta \mathcal{L}^{\text{osr}}\Big)^{\!\top} 
\nabla_\Theta \mathcal{L}^{\text{align}} \;\ge\; 0.
\end{equation}
Consequently, for sufficiently small $\delta$, $\mathcal{L}^{\text{align}}(\Theta+\Delta\Theta)$ does not increase to first order.
\end{theorem}

\noindent\textbf{Implication via Taylor expansion.}
By a first-order Taylor approximation \citep{rudin1976principles},
\begin{align}
\label{eq:taylor}
\mathcal{L}^{\text{align}}(\Theta+\Delta\Theta) 
&\approx \mathcal{L}^{\text{align}}(\Theta) 
+ \nabla_\Theta \mathcal{L}^{\text{align}}(\Theta)^{\!\top}\Delta\Theta \nonumber\\
&= \mathcal{L}^{\text{align}}(\Theta) 
-\delta\Big(\nabla_\Theta \mathcal{L}^{\text{align}} + \lambda \nabla_\Theta \mathcal{L}^{\text{osr}}\Big)^{\!\top}
\nabla_\Theta \mathcal{L}^{\text{align}}.
\end{align}
Under Eq.~\ref{eq:nonnegative_inner_product}, the last term is non-positive, hence $\mathcal{L}^{\text{align}}$ does not increase to first order. 
Equivalently, the anchor similarity $\text{sim}(x_i^b,c_i)$ is preserved (or improved) locally under the combined update.

\paragraph{A concrete corollary for cosine similarity: why $\lambda \le 1$ is a stable regime.}
We now instantiate the above condition under the common choice of cosine similarity with $\ell_2$-normalized embeddings.
When $\text{sim}(u,v)=u^\top v$ and $\|x_i^b\|=\|c_i\|=\|\hat x_i^a\|=1$, the gradient of the alignment term w.r.t.\ the embedding satisfies
\begin{equation}
\left\|-\frac{\partial \mathcal{L}^{\text{align}}_i}{\partial x_i^b}\right\|
=
\left\|\frac{1}{\tau}\frac{\partial (c_i^\top x_i^b)}{\partial x_i^b}\right\|
=
\frac{1}{\tau}\|c_i\|
=
\frac{1}{\tau}.
\end{equation}
Similarly, for the positive-pair part of OSR in the orthogonal subspace, the gradient magnitude w.r.t.\ the normalized representation $T_{\bar c_i}(x_i^b)$ is
\begin{equation}
\left\|-\frac{\partial}{\partial T_{\bar c_i}(x_i^b)}\left(\frac{\hat x_i^{a\top} T_{\bar c_i}(x_i^b)}{\tau}\right)\right\|
=
\frac{1}{\tau}\|\hat x_i^a\|
=
\frac{1}{\tau}.
\end{equation}
Therefore, the relative gradient scales are comparable, and Eq.~\ref{eq:lambda_condition_general} admits the intuitive stable regime $\lambda \le 1$.
Empirically, we observe peak zero-shot performance near this regime, consistent with the theory.

\paragraph{On the negative (log-partition) term.}
The above analysis focuses on the positive-pair alignment component to highlight the new effect introduced by OSR.
The negative (log-partition) term can also affect similarity through interactions with in-batch negatives and sampling dynamics. 
Such effects have been analyzed extensively for standard contrastive learning objectives \citep{mindthegap,understandingInfoNCE}. 
Since EmergentBridge leaves the anchor InfoNCE objective unchanged and only constrains the \emph{proxy-driven} update via orthogonalization, the negative-term behavior follows the same fundamental principles as in prior work; we therefore omit a redundant analysis and focus on the orthogonal-subspace mechanism unique to EmergentBridge.
\begin{table*}[t]
\centering
\small
\setlength{\tabcolsep}{5pt}
\caption{\textbf{X--Language classification.} 
$^{*}$ denotes \textbf{emergent zero-shot} (image anchor). 
We report top-1 accuracy (\%) for all datasets except AudioSet (mAP). Best results are in bold.}
\begin{tabular}{l|c|c|cc|ccc}
\toprule
\textbf{Method} & \textbf{Anchor} & \textbf{Infrared} & \multicolumn{2}{c|}{\textbf{Depth}} & \multicolumn{3}{c}{\textbf{Audio}} \\
 & \textbf{Modality} & LLVIP & NYU-D & SUN & AudioSet & ESC-50 & VGG-S \\
\midrule
OpenCLIP \citep{cherti2023reproducible} & -- & 82.2 & 45.4 & 25.4 & -- & -- & -- \\
DepthSwin \citep{girdhar2022omnivore}   & -- & --   & 72.5 & 63.1 & -- & -- & -- \\
JointCRF \citep{wang2017learning}       & -- & --   & 65.8 & 63.6 & -- & -- & -- \\
DFCR \citep{8010878}                    & -- & --   & 65.3 & 56.3 & -- & -- & -- \\
AudioCLIP \citep{guzhov2021audioclipextendingclipimage} & -- & -- & -- & -- & 28.4 & 68.6 & 47.4 \\
CLAP \citep{elizalde2023clap}           & -- & --   & --   & --   & 23.1 & 92.6 & 46.2 \\
WAV2CLIP \citep{Wu2021Wav2CLIPLR}       & -- & --   & --   & --   & 0.71 & 41.4 & 10.0 \\
\midrule
LanguageBind            & Text  & 87.2 & 65.1 & --   & 27.7 & 91.8 & 28.9 \\
\textbf{EmergentBridge} & Text  & \textbf{85.1} & \textbf{65.8} & -- & \textbf{28.1} & \textbf{92.0} & \textbf{29.3} \\
\midrule
ImageBind$^{*}$                 & Image & 63.4 & 54.0 & 35.1 & 17.6 & 66.9 & 27.8 \\
\textbf{EmergentBridge}$^{*}$   & Image & \textbf{73.7} & \textbf{70.1} & \textbf{40.3} & \textbf{25.4} & \textbf{68.4} & \textbf{36.3} \\
\midrule
Absolute SOTA & -- & -- & 79.4 & 64.9 & 49.6 & 97.0 & 52.5 \\
\bottomrule
\end{tabular}
\label{tab:zero-shot x-language}
\end{table*}

\begin{figure*}
    \centering
    \includegraphics[width=0.95\linewidth]{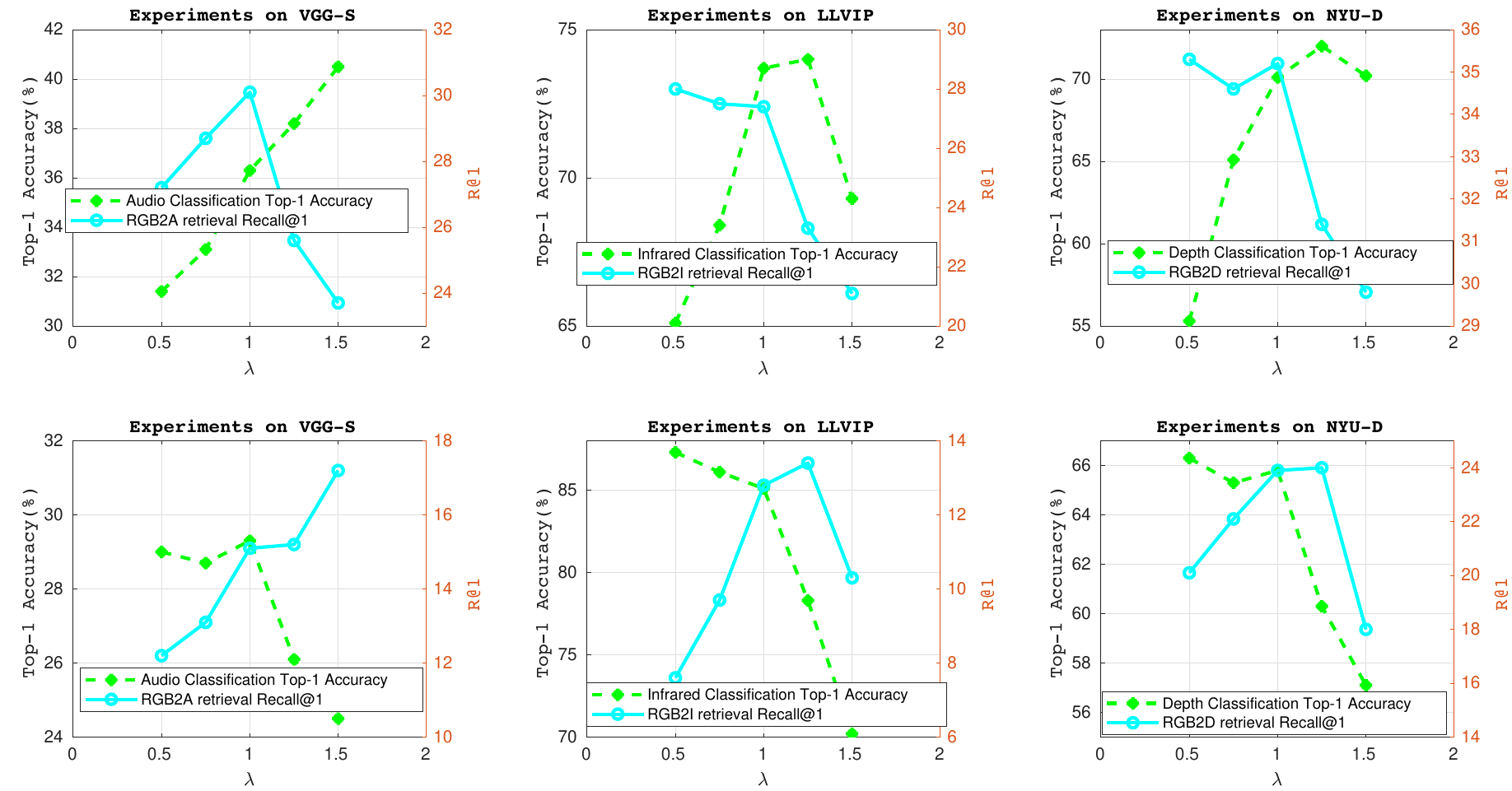}
    \caption{\textbf{Experimental Results on Classification and Retrieval Tasks After Varying Hyperparameters $\lambda$:} The first row of figures, arranged from left to right, presents the results on the VGG-S, LLVIP, and NYU-D datasets, with images serving as the core modality. The second row of figures, similarly ordered, displays the results on the same datasets, but with text serving as the core modality.
    }
    \label{fig:changelambda}
\end{figure*}

\begin{table}[t]
\centering
\small
\setlength{\tabcolsep}{4pt}
\caption{\textbf{Zero-shot audio--language retrieval.} 
$^{*}$ denotes \textbf{emergent zero-shot} (image anchor). We report Recall@K.}
\begin{tabular}{l|c|cc|cc}
\toprule
\multirow{2}{*}{\textbf{Method}} & \textbf{Anchor} & \multicolumn{2}{c|}{\textbf{Clotho}} & \multicolumn{2}{c}{\textbf{AudioCaps}} \\
& \textbf{Modality} & R@1 & R@10 & R@1 & R@10 \\
\midrule
AVFIC \citep{AVFIC} & -- & 3.0 & 17.5 & 8.7 & 37.7 \\
AudioCLIP \citep{guzhov2021audioclipextendingclipimage} & -- & 3.20 & 20.3 & 3.53 & 31.6 \\
WAV2CLIP \citep{Wu2021Wav2CLIPLR} & -- & 0.78 & 12.1 & 0.88 & 15.3 \\
C-MCR \citep{C-MCR} & -- & 8.37 & 36.7 & 15.76 & 48.1 \\
\midrule
LanguageBind & Text & \textbf{12.1} & \textbf{44.0} & \textbf{12.2} & \textbf{53.2} \\
EmergentBridge & Text & 11.7 & 41.8 & 11.8 & 52.1 \\
\midrule
ImageBind$^{*}$ & Image & 6.0 & 28.4 & 9.3 & 42.3 \\
EmergentBridge$^{*}$ & Image & \textbf{10.0} & \textbf{33.4} & \textbf{10.1} & \textbf{49.2} \\
\bottomrule
\end{tabular}
\label{tab:Zero-shot Audio-Language retrieval}
\end{table}

\begin{table}[t]
\centering
\small
\setlength{\tabcolsep}{3pt}
\renewcommand{\arraystretch}{1.05}
\caption{\textbf{RGB$\rightarrow$X retrieval.}
$^{*}$ denotes the \textbf{emergent} setting with text as the anchor; we report Recall@1. Best results are in bold.}
\begin{tabular*}{\linewidth}{@{\extracolsep{\fill}} l l c r @{}}
\toprule
\textbf{Dataset} & \textbf{Method} & \textbf{Anchor} & \textbf{R@1} \\
\midrule
\multirow{4}{*}{AVE}
& ImageBind                    & Image & 36.9 \\
& \textbf{EmergentBridge}      & Image & \textbf{37.0} \\
& LanguageBind$^{*}$           & Text  & 10.6 \\
& \textbf{EmergentBridge}$^{*}$& Text  & \textbf{15.1} \\
\midrule
\multirow{4}{*}{VGG-S}
& ImageBind                    & Image & 28.7 \\
& \textbf{EmergentBridge}      & Image & \textbf{30.1} \\
& LanguageBind$^{*}$           & Text  & 10.0 \\
& \textbf{EmergentBridge}$^{*}$& Text  & \textbf{15.1} \\

\midrule
\multirow{4}{*}{LLVIP}
& ImageBind                    & Image & 26.3 \\
& \textbf{EmergentBridge}      & Image & \textbf{27.4} \\
& LanguageBind$^{*}$           & Text  & 7.5 \\
& \textbf{EmergentBridge}$^{*}$& Text  & \textbf{12.8} \\

\midrule
\multirow{4}{*}{NYU-D}
& ImageBind                    & Image & 34.7 \\
& \textbf{EmergentBridge}      & Image & \textbf{35.2} \\
& LanguageBind$^{*}$           & Text  & 17.9 \\
& \textbf{EmergentBridge}$^{*}$& Text  & \textbf{23.9} \\
\bottomrule
\end{tabular*}
\label{tab: Cross zero-shot retrieval}
\vspace{-3mm}
\end{table}

\begin{table*}[t]
\centering
\small
\setlength{\tabcolsep}{4pt}
\caption{\textbf{Comparison of proxy embedding synthesis methods (Step~2).}
We instantiate the synthesis network $N_\Theta$ with different architectures to generate proxy embeddings $\hat{x}^a$ from anchor embeddings $c$.
We report top-1 accuracy (\%) on $X\!\rightarrow\!T$ and Recall@1 (\%) on $\mathrm{RGB}\!\rightarrow\!X$.
Diff denotes conditional diffusion-based synthesis. 
$^{*}$ marks \textbf{emergent zero-shot} tasks (i.e., unpaired modality pairs). Best results are in bold.}
\begin{tabular}{c|c|ccccc|ccccc}
\toprule
\multirow{2}{*}{\textbf{Dataset}} & \multirow{2}{*}{\textbf{Task}} 
& \multicolumn{5}{c|}{\textbf{Image Anchor}} 
& \multicolumn{5}{c}{\textbf{Text Anchor}} \\
& & Diff & ResNet & VAE & Noise & C-MCR & Diff & ResNet & VAE & Noise & C-MCR \\
\midrule
\multirow{2}{*}{VGG-S} 
& A$\rightarrow$T$^{*}$  & \textbf{36.3} & 30.7 & 32.1 & 23.7 & 26.7 & \textbf{29.3} & 27.2 & 28.3 & 26.4 & 29.1 \\
& RGB$\rightarrow$A      & \textbf{30.1} & 29.3 & 28.1 & 27.7 & 28.3 & \textbf{15.1}$^{*}$ & 12.3$^{*}$ & 13.5$^{*}$ & 8.1$^{*}$ & 9.5$^{*}$ \\
\midrule
\multirow{2}{*}{LLVIP} 
& I$\rightarrow$T$^{*}$  & \textbf{73.7} & 68.1 & 64.3 & 59.1 & 61.7 & 85.1 & 83.1 & \textbf{87.5} & 86.2 & 86.9 \\
& RGB$\rightarrow$I      & \textbf{27.4} & 26.5 & 27.1 & 26.1 & 25.9 & \textbf{12.8}$^{*}$ & 10.3$^{*}$ & 9.1$^{*}$ & 5.2$^{*}$ & 8.9$^{*}$ \\
\midrule
\multirow{2}{*}{NYU-D} 
& D$\rightarrow$T$^{*}$  & \textbf{70.1} & 67.2 & 69.3 & 50.2 & 58.4 & \textbf{65.8} & 65.7 & 64.8 & 64.1 & 65.4 \\
& RGB$\rightarrow$D      & \textbf{35.2} & 35.0 & 33.9 & 33.1 & 34.4 & \textbf{23.9}$^{*}$ & 20.7$^{*}$ & 22.4$^{*}$ & 14.6$^{*}$ & 18.9$^{*}$ \\
\bottomrule
\end{tabular}
\label{tab:ab gen net}
\end{table*}

\begin{table}[t]
\centering
\small
\setlength{\tabcolsep}{4pt}
\caption{\textbf{Replacing OSR with direct InfoNCE for proxy alignment.}
We replace OSR by removing the projection $T_{\bar{c}_i}(\cdot)$ and directly applying InfoNCE between $x^b$ and the proxy embedding $\hat{x}^a$ (keeping $\lambda{=}1$).
We report Recall@1 for image-anchored retrieval and top-1 accuracy (\%) for text-anchored classification.}
\begin{tabular}{c|c|c|cc}
\toprule
\multirow{2}{*}{\textbf{Dataset}} & \multirow{2}{*}{\textbf{Task}} & \multirow{2}{*}{\textbf{Anchor}} & \multicolumn{2}{c}{\textbf{Proxy Alignment}}\\
& & & \textbf{OSR} & \textbf{Direct InfoNCE}\\
\midrule
\multirow{2}{*}{VGG-S}  & RGB$\rightarrow$A & Image & \textbf{30.1} & 25.2 \\
                        & A$\rightarrow$T   & Text  & \textbf{29.3} & 22.2 \\
\midrule
AVE                    & RGB$\rightarrow$A & Image & \textbf{37.0} & 32.3 \\
\midrule
\multirow{2}{*}{NYU-D}  & RGB$\rightarrow$D & Image & \textbf{35.2} & 29.1 \\
                        & D$\rightarrow$T   & Text  & \textbf{65.8} & 55.7 \\
\midrule
\multirow{2}{*}{LLVIP}  & RGB$\rightarrow$I & Image & \textbf{27.4} & 22.4 \\
                        & I$\rightarrow$T   & Text  & \textbf{85.1} & 64.1 \\
\midrule
ESC-50                 & A$\rightarrow$T   & Text  & \textbf{92.0} & 72.3 \\
\bottomrule
\end{tabular}
\label{tab:x3}
\end{table}

\section{Experiments and Results}
\label{EXPERIMENTS}

We evaluate \textbf{EmergentBridge} on two aspects: (i) \emph{emergent} zero-shot transfer over \emph{unpaired} modality pairs, and (ii) preservation of performance on anchor-supervised pairs. We further conduct ablations to isolate the effects of proxy embedding synthesis (Step~\ref{step2}), the regularization weight $\lambda$, and the proposed orthogonal-subspace regularizer (Step~\ref{step3}).

\subsection{Settings}

\paragraph{Modalities and Datasets.}
We consider $\mathcal{C}\in\{\text{image},\text{text}\}$ as the \emph{anchor} modality, as these are the only modalities paired with multiple others across our benchmarks. For each non-anchor modality $\mathcal{M}_b\in\{\text{audio},\text{depth},\text{infrared}\}$, we train EmergentBridge using AudioSet \citep{AudioCaps}, SUN \citep{SUN-D}, and LLVIP \citep{LLVIP}, respectively. All remaining datasets are held out for zero-shot evaluation; detailed dataset descriptions are provided in Appendix~\ref{Downstream Datasets}.

\paragraph{Implementation Details.}
We use cosine similarity $\text{sim}(\cdot,\cdot)$ throughout. The anchor encoder for $\mathcal{C}$ is initialized from ImageBind-Huge \citep{girdhar2023imagebindembeddingspacebind} or LanguageBind \citep{zhu2024languagebindextendingvideolanguagepretraining} and kept frozen during training. 
We first align an already-aligned modality $\mathcal{M}_a$ to $\mathcal{C}$ (Step~\ref{step1}), then train the synthesis network $N_\Theta$ to map anchor embeddings $c_i$ to proxy embeddings $\hat{x}^a_i$ (Step~\ref{step2}), and finally train the encoder of the target modality $\mathcal{M}_b$ to produce $x_i^b$ aligned with both $c_i$ and $\hat{x}^a_i$ under the joint objective in Eq.~\ref{eq:overall_loss} (Step~\ref{step3}). 
Unless otherwise specified, we instantiate $N_\Theta$ with a conditional diffusion model to model the conditional distribution $p(x^a\!\mid\!c)$ rather than a point estimate, which is beneficial when the cross-modal embedding correspondence is multi-modal. For training and evaluation, we use a deterministic sampling rule to obtain a single, stable proxy embedding $\hat{x}^a_i$ for each $c_i$, and empirically validate this choice against deterministic regressors in  Sec.\ref{The effect of generative networks}.
Additional implementation details are provided in Appendix~\ref{Implementation Details}.

\subsection{Main Results}

\paragraph{Emergent Zero-shot Performance.}
EmergentBridge substantially improves \emph{emergent} zero-shot transfer on \emph{unpaired} modality pairs under both image-anchor and text-anchor settings (Tables~\ref{tab:zero-shot x-language}--\ref{tab: Cross zero-shot retrieval}). 
With \textbf{image} as the anchor, EmergentBridge consistently improves emergent $X\!\rightarrow\!T$ classification over ImageBind on six datasets (Table~\ref{tab:zero-shot x-language}), with notable top-1 gains of 10.3\%, 16.1\%, and 8.5\% on LLVIP, NYU-D, and VGG-S, respectively. Overall, it achieves an average \emph{relative} improvement of 24.7\% on emergent classification. 
For emergent audio--language retrieval (Table~\ref{tab:Zero-shot Audio-Language retrieval}), EmergentBridge improves ImageBind by 5.0\% and 6.9\% in R@10 on Clotho and AudioCaps. 

With \textbf{text} as the anchor, EmergentBridge improves LanguageBind on emergent RGB$\rightarrow X$ retrieval (Table~\ref{tab: Cross zero-shot retrieval}), yielding R@1 gains of 4.5\%, 5.1\%, 5.3\%, and 6.0\% on AVE, VGG-S, LLVIP, and NYU-D, respectively. On average, EmergentBridge improves emergent retrieval by 49.4\% (\emph{relative}) over prior binding baselines. 
These results indicate that strengthening cross-modality connectivity via proxy bridging translates into consistently better zero-shot transfer on unseen modality pairs; the specific contribution of the orthogonal-subspace regularizer is further isolated in the ablation study.

\paragraph{Anchor-Alignment Preservation.}
Beyond emergent transfer, we examine whether EmergentBridge preserves performance on anchor-supervised pairs that define the original embedding space. 
With \textbf{image} as the anchor, EmergentBridge matches or slightly improves ImageBind on RGB$\rightarrow X$ retrieval (Table~\ref{tab: Cross zero-shot retrieval}), with small R@1 gains of 0.1\%, 1.3\%, and 0.5\% on AVE, VGG-S, and NYU-D, respectively. 
With \textbf{text} as the anchor, EmergentBridge yields modest improvements over LanguageBind on $X\!\rightarrow\!T$ classification for VGG-S, NYU-D, and ESC-50 (Table~\ref{tab:zero-shot x-language}), while incurring minor degradations on LLVIP and the audio--language retrieval benchmarks (Table~\ref{tab:Zero-shot Audio-Language retrieval}). 
Overall, EmergentBridge improves emergent transfer while keeping anchor-aligned performance broadly stable; the interference-reduction role of the orthogonal-subspace regularizer is validated by the controlled replacement study in \S\ref{WithoutTangentTerm}.

\subsection{Ablation Study}

\paragraph{Comparison of Embedding Synthesis Methods.}
\label{The effect of generative networks}

This ablation examines how the choice of the synthesis network in Step~2 affects emergent transfer. 
Recall that EmergentBridge learns a proxy generator $N_\Theta$ that maps anchor embeddings $c_i$ to proxy embeddings $\hat{x}^a_i$, which are then used to bridge an unpaired modality in Step~3. 
By default, we instantiate $N_\Theta$ as a \emph{conditional diffusion} model to better capture potentially multi-modal mappings from $c$ to $x^a$; in all experiments we use a deterministic sampling rule to obtain a single stable proxy embedding for each $c_i$ (details in Appendix~\ref{Implementation Details}).

To assess whether simpler alternatives suffice, we replace $N_\Theta$ with deterministic predictors (ResNet \citep{ResNet}, VAE \citep{VAE}, and C-MCR \citep{C-MCR}) and a perturbation baseline that adds Gaussian noise ($\mathcal{N}(0,10^{-3})$) to the anchor embeddings before normalization. 
Table~\ref{tab:ab gen net} shows that diffusion-based synthesis yields the strongest and most consistent gains on \emph{emergent zero-shot} tasks across datasets. 
For example, on emergent $X\!\rightarrow\!T$ classification, diffusion achieves 36.3\%, 73.7\%, and 70.1\% top-1 accuracy on VGG-S, LLVIP, and NYU-D, respectively, outperforming the alternatives by clear margins; similar trends hold for emergent $\mathrm{RGB}\!\rightarrow\!X$ retrieval. 
Appendix~\ref{cdfanalysis} further supports this result by showing that diffusion proxies exhibit higher cosine similarity to real embeddings, indicating better fidelity in capturing cross-modal structure.

\paragraph{\texorpdfstring{Effect of Hyperparameter $\lambda$.}{Effect of Hyperparameter lambda.}}
\label{changelambda}

Figure~\ref{fig:changelambda} studies the trade-off controlled by $\lambda$ between improving \emph{emergent} transfer (green curves: unpaired $X\!\rightarrow\!T$ classification) and preserving anchor-aligned utility (blue curves: $\mathrm{RGB}\!\rightarrow\!X$ retrieval R@1). 
Across datasets and for both anchor choices (image in the first row, text in the second), increasing $\lambda$ from a small value consistently strengthens emergent zero-shot performance, indicating that the proxy-alignment term provides complementary cross-modal signals. 
However, overly large $\lambda$ can degrade retrieval, and in several cases also causes emergent performance to plateau or drop, suggesting increased interference when proxy-driven updates dominate.

This behavior aligns with our theoretical analysis (Theorem~\ref{thm:main}) in that it provides a \emph{sufficient} regime where the orthogonal-subspace regularizer does not disrupt anchor-alignment descent to first order. 
Empirically, the best performance typically occurs around $\lambda\!\approx\!1$ (and can be slightly larger on some datasets), which is plausible given effects not captured by the simplified analysis, such as the negative/log-partition term $\mathcal{L}^{\text{uniform}}$, finite-batch optimization, and residual proxy noise.

\paragraph{Replacing the Orthogonal-Subspace Regularizer with Direct InfoNCE}
\label{WithoutTangentTerm}

We isolate the effect of the orthogonal-subspace regularizer (OSR) by replacing the proxy-alignment term with a standard contrastive objective. 
Concretely, we remove the orthogonal projection $T_{\bar{c}_i}(\cdot)$ in Eq.~\ref{eq:tanloss} and instead align $x_i^b$ to the synthesized proxy $\hat{x}_i^a$ using a direct InfoNCE loss. 
That is, we replace $\mathcal{L}^{\text{tan}}_{\mathcal{M}_b\rightarrow\mathcal{M}_a}+\mathcal{L}^{\text{tan}}_{\mathcal{M}_a\rightarrow\mathcal{M}_b}$ with $\mathcal{L}^{\text{infoNCE}}_{\mathcal{M}_b\rightarrow\mathcal{M}_a}+\mathcal{L}^{\text{infoNCE}}_{\mathcal{M}_a\rightarrow\mathcal{M}_b}$, while keeping $\lambda=1$.

As shown in Table~\ref{tab:x3}, direct proxy alignment leads to substantial drops across all evaluated settings. 
For image-anchored retrieval (AVE/VGG-S/NYU-D/LLVIP), R@1 decreases noticeably when replacing OSR with direct InfoNCE. 
For text-anchored classification (VGG-S/ESC-50/NYU-D/LLVIP), top-1 accuracy also degrades sharply (e.g., 29.3$\rightarrow$22.2 on VGG-S and 92.0$\rightarrow$72.3 on ESC-50). 
These results support our motivation that a noisy proxy can induce \emph{gradient interference}: directly pulling $x_i^b$ toward $\hat{x}_i^a$ may compete with the anchor-alignment objective, whereas OSR constrains proxy-driven updates to the orthogonal subspace and thereby mitigates such interference while strengthening emergent transfer.
\section{Conclusion}

In this work, we introduce \textbf{EmergentBridge}, a framework for strengthening \emph{emergent alignment} in unified multimodal embedding models when supervision is available for only a small subset of modality pairs.
By operating at the embedding level, EmergentBridge enables practical expansion to new modalities without curating exhaustive pairwise data or re-pretraining large backbones.

The central insight of our approach is to bridge indirectly related modalities through synthesized proxy embeddings, allowing a new modality to align jointly with an anchor and an already-aligned modality.
To ensure robustness, we further introduce an orthogonal-subspace regularizer that mitigates interference from imperfect proxy signals while preserving well-supervised anchor alignments.
Across nine datasets and multiple modality pairs, EmergentBridge consistently improves zero-shot transfer in retrieval and classification, while maintaining strong performance on anchor-supervised tasks.

More broadly, our results suggest that explicitly modeling relationships between unpaired modalities is critical for scalable and reliable multimodal learning.
We hope this work encourages future research on embedding-level mechanisms for cross-modal generalization under sparse and heterogeneous supervision.
Future work will explore stronger proxy synthesis objectives and adaptive regularization strategies that further improve robustness under larger domain gaps and noisier supervision.

\bibliographystyle{ACM-Reference-Format}
\bibliography{bibfile}

\appendix

\section{Implementation Details}
\label{Implementation Details}
Tables~\ref{image set} and~\ref{text set} summarize the key hyperparameters used in our experiments, including optimizer settings and modality-specific temperatures.
All experiments were conducted on 4$\times$24GB NVIDIA RTX 4090 GPUs and 4$\times$48GB NVIDIA A40 GPUs.
\subsection{Prompt Templates}
\label{Prompt Templates}
For all evaluations, we adopt the default set of templates provided by CLIP \citep{radford2021learningtransferablevisualmodels}. Notably, the same templates are utilized for non-visual modalities, such as audio and depth, as the training process relies solely on the semantic or textual supervision associated with images.
For non-visual modalities (e.g., audio and depth), we follow the same template set, since supervision in our setting is mediated through the anchor modality and text semantics in the shared embedding space.
\subsection{Model Architecture}
\label{app:model_arch}

\paragraph{Diffusion-based Proxy Predictor $N_\Theta$ (Embedding Synthesis).}
We instantiate the proxy synthesis network $N_\Theta$ as a \emph{conditional diffusion} model operating entirely in the embedding space. Concretely, we train two predictors: (i) an image-conditioned model that maps image embeddings to proxy text embeddings for the \textbf{image-anchor} setting, and (ii) a text-conditioned model that maps text embeddings to proxy image embeddings for the \textbf{text-anchor} setting. Both predictors share the same Transformer backbone: 6 layers, 8 attention heads (128 dimensions per head), and a 1024-dimensional embedding width, totaling $\sim$30M parameters. 

Training follows the proxy regression objective in Eq.~\ref{eq:genloss}. We use 100 diffusion steps and apply 5\% embedding dropout for classifier-free guidance. Optimization uses AdamW with learning rate $1\times 10^{-4}$ and batch size 128. During training and evaluation, we use a deterministic sampling rule to produce a single stable proxy embedding for each conditioning anchor embedding (see Appendix~\ref{Implementation Details} for the sampling configuration). This design provides the proxy bridge required by EmergentBridge \emph{without} generating raw data.

\paragraph{Multimodal Encoders.}
For both anchor choices, we use the same encoder architecture for each non-anchor modality. Following \citet{girdhar2023imagebindembeddingspacebind}, we adopt a 12-layer Vision Transformer (ViT) with 1024-dimensional hidden size and patch size 16 (stride 10) to encode vision-like inputs for \textsc{vision}, \textsc{depth}, and \textsc{infrared}. Depth and thermal/infrared inputs are treated as single-channel images and processed by the same ViT architecture to ensure consistent capacity across modalities. For audio, we convert 2-second clips sampled at 16~kHz into log-mel spectrograms with 128 mel bins and feed them to the audio encoder.

Anchor encoders are initialized from ImageBind-Huge and LanguageBind (for image and text anchors, respectively) and kept frozen throughout training. We train the proxy predictor $N_\Theta$ and fine-tune the encoder of each newly added modality under the EmergentBridge objective in Eq.~\ref{eq:overall_loss}.

\paragraph{Temperature.}
The InfoNCE loss in Eq.~\ref{eq:infoNCE} and the orthogonal-subspace proxy-alignment loss in Eq.~\ref{eq:tanloss} share the same temperature $\tau$ for a given modality during training. We found fixed temperatures to work better than learnable ones. The modality-specific temperatures used in our experiments are reported in Tables~\ref{image set} and~\ref{text set}.

\begin{table}[]
\caption{\textbf{Training Settings in ImageCore Model}}
\centering
\vspace{2mm}
\resizebox{0.48\textwidth}{!}{\begin{tabular}{c|ccc}
Config                 & \multicolumn{1}{c|}{\textbf{Audio}} & \multicolumn{1}{c|}{\textbf{Depth}} & \textbf{Infrared} \\ \hline
Encoder                & \multicolumn{3}{c}{ViT-Huge}                                                                  \\
Number of Heads        & \multicolumn{1}{c|}{12}             & \multicolumn{1}{c|}{8}              & 12                \\
Optimizer              & \multicolumn{3}{c}{AdamW}                                                                     \\
Optimizer Momentum     & \multicolumn{3}{c}{$\beta_1 = 0.9, \beta_2 = 0.95$}                                           \\
Epochs                 & \multicolumn{1}{c|}{8}              & \multicolumn{1}{c|}{2}              & 2                 \\
Learning rate          & \multicolumn{1}{c|}{5e-4}           & \multicolumn{1}{c|}{5e-4}           & 1e-4              \\
Temperature            & \multicolumn{1}{c|}{0.07}           & \multicolumn{1}{c|}{0.2}            & 0.1               \\
Weight decay           & \multicolumn{1}{c|}{0.2}            & \multicolumn{1}{c|}{0.2}            & 0.05              \\
Batch size             & \multicolumn{1}{c|}{512}            & \multicolumn{1}{c|}{256}            & 256               \\ 
Learning rate schedule & \multicolumn{3}{c}{Cosine decay}      
\end{tabular}}
\label{image set}
\end{table}

\begin{table}[H]
\caption{\textbf{Training Settings in LanguageCore Model}}
\centering
\vspace{2mm}
\resizebox{0.48\textwidth}{!}{
\begin{tabular}{c|ccc}
Config                 & \multicolumn{1}{c|}{\textbf{Audio}} & \multicolumn{1}{c|}{\textbf{Depth}} & \textbf{Infrared} \\ \hline
Encoder                & \multicolumn{3}{c}{ViT-Huge}                                                                  \\
Number of Heads        & \multicolumn{1}{c|}{12}             & \multicolumn{1}{c|}{8}              & 12                \\
Optimizer              & \multicolumn{3}{c}{AdamW}                                                                     \\
Optimizer Momentum     & \multicolumn{3}{c}{$\beta_1 = 0.9, \beta_2 = 0.95$}                                           \\
Epochs                 & \multicolumn{1}{c|}{8}              & \multicolumn{1}{c|}{4}              & 4                 \\
Learning rate          & \multicolumn{1}{c|}{1e-4}           & \multicolumn{1}{c|}{5e-4}           & 1e-4              \\
Temperature            & \multicolumn{1}{c|}{0.05}           & \multicolumn{1}{c|}{0.2}            & 0.2               \\
Weight decay           & \multicolumn{1}{c|}{0.2}            & \multicolumn{1}{c|}{0.1}            & 0.05              \\
Batch size             & \multicolumn{1}{c|}{512}            & \multicolumn{1}{c|}{256}            & 256               \\
Learning rate schedule & \multicolumn{3}{c}{Cosine decay}   
\end{tabular}}
\label{text set}
\end{table}

\begin{figure*}[h]
    \centering
    \includegraphics[width=0.95\linewidth]{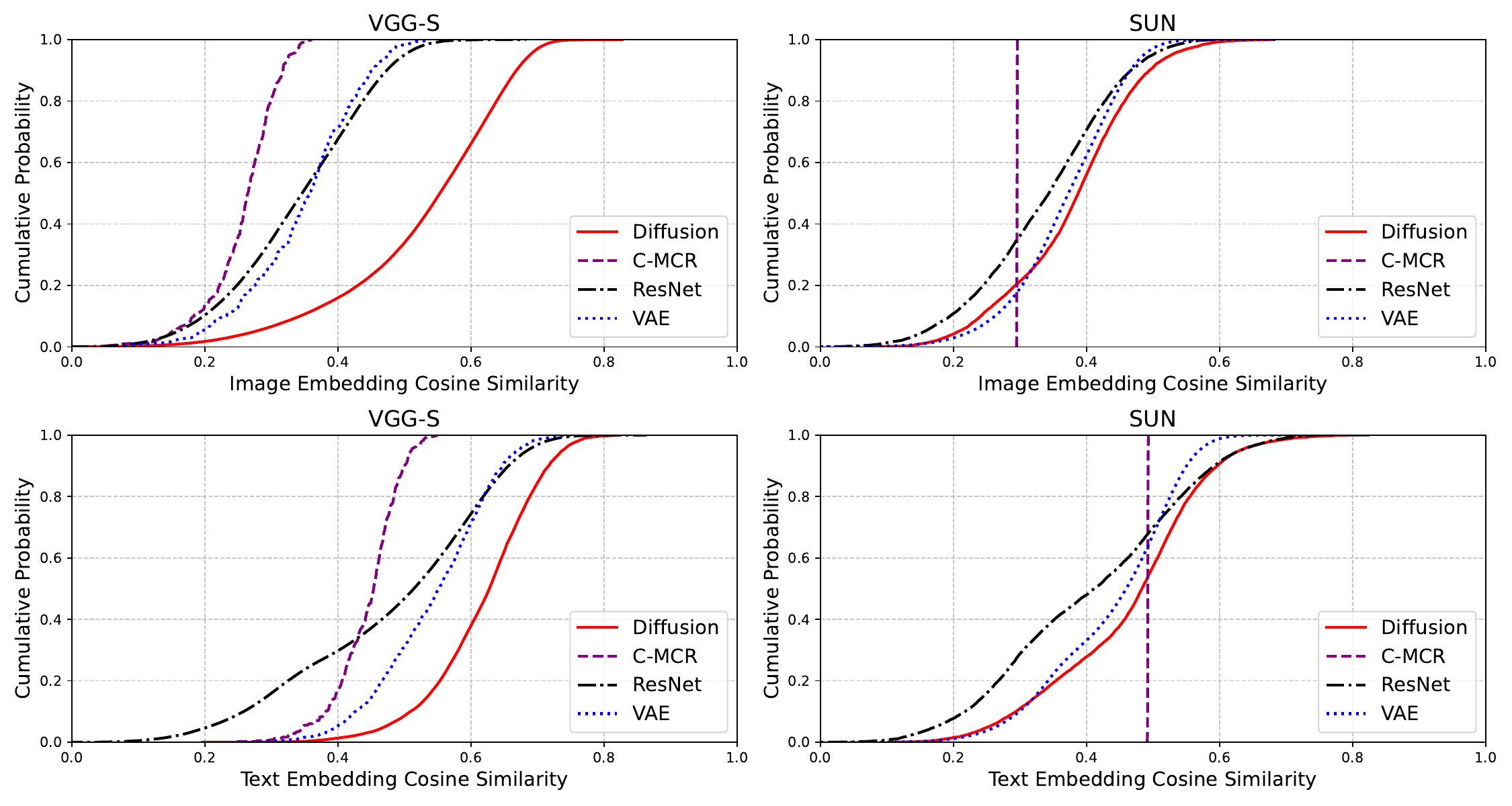}
    \caption{\textbf{CDF Curves of Similarity on VGG-S and SUN:}
The top and bottom rows respectively illustrate the cumulative distribution functions (CDFs) of cosine similarity between generated text/image embeddings (produced by different methods in Step \ref{step2}) and the corresponding actual text/image embeddings. Results are presented with image and text as the core modalities.
    The solid, dashed, dash-dotted, and dotted lines correspond to the diffusion, C-MCR, ResNet, and VAE methods, respectively.}
    \label{fig:generative CDF}
\end{figure*}

\section{Ablation Details}
\label{Ablation Details}

This section provides implementation details for the alternative proxy synthesis methods used in \S\ref{The effect of generative networks}. 
All methods below are \emph{plug-in replacements} for the synthesis network $N_\Theta$ in Step~2: given an anchor embedding $c_i\in\mathcal{C}$, they predict a proxy embedding in the paired modality (e.g., $x_i^{a}$) and output a normalized proxy $\hat{x}_i^{a}=\text{normalize}(N_\Theta(c_i))$. 
To ensure a controlled comparison, we train all predictors on the same paired embedding set (the training split of ImageNet-1K-VL-Enriched) with the same regression objective in Eq.~\ref{eq:genloss}, and we keep the optimizer settings consistent (AdamW, learning rate $1\times 10^{-4}$, batch size 128).

\subsection{ResNet Regressor}
We use the standard ResNet-50 implementation from PyTorch \citep{ResNet} as a deterministic regressor. 
We replace the final classification layer with a 1024-dimensional linear head to match the embedding dimension. 
The predicted embedding is $\ell_2$-normalized to the unit hypersphere before evaluation. 
We optimize the proxy regression objective (Eq.~\ref{eq:genloss}) using AdamW (learning rate $1\times 10^{-4}$, batch size 128). 
Inputs are anchor embeddings $c_i$, and targets are paired modality embeddings $x_i^{a}$ extracted from the corresponding frozen encoders.

\subsection{C-MCR Memory Retrieval}
We implement C-MCR \citep{C-MCR} as a non-parametric proxy synthesis baseline based on cross-modal memory retrieval. 
We construct a paired memory $\{(m_k^{\text{img}}, m_k^{\text{text}})\}_{k=1}^{N}$ from ImageNet-1K-VL-Enriched embeddings. 
Given a query embedding $q^{\text{text}}$ (or $q^{\text{img}}$), C-MCR synthesizes the cross-modal proxy via a softmax-weighted sum:
\begin{align}
\hat{x}^{\text{img}}(q^{\text{text}}) &= 
\text{normalize}\!\left(\sum_{k=1}^{N} \alpha_k(q^{\text{text}})\, m_k^{\text{img}}\right), \\
\hat{x}^{\text{text}}(q^{\text{img}}) &= 
\text{normalize}\!\left(\sum_{k=1}^{N} \beta_k(q^{\text{img}})\, m_k^{\text{text}}\right),
\end{align}
where
\begin{align}
\alpha_k(q^{\text{text}}) &= 
\frac{\exp(\text{sim}(q^{\text{text}}, m_k^{\text{text}})/\tau)}
{\sum_{\ell=1}^{N}\exp(\text{sim}(q^{\text{text}}, m_{\ell}^{\text{text}})/\tau)}, \\
\beta_k(q^{\text{img}}) &= 
\frac{\exp(\text{sim}(q^{\text{img}}, m_k^{\text{img}})/\tau)}
{\sum_{\ell=1}^{N}\exp(\text{sim}(q^{\text{img}}, m_{\ell}^{\text{img}})/\tau)}.
\end{align}
In our CDF analysis (Fig.~\ref{fig:generative CDF}), we observe that on SUN the retrieved proxies can collapse toward near-constant embeddings. 
This is consistent with a strong domain gap between SUN (scene-centric) and the ImageNet-1K-VL-Enriched memory, which makes similarity scores less discriminative and yields near-uniform softmax weights, thus reducing the diversity of synthesized proxies. 
This behavior highlights C-MCR's reliance on the memory distribution.

\subsection{VAE Regressor}
We implement a Variational Autoencoder (VAE) \citep{VAE} as a deterministic encoder--decoder baseline for proxy synthesis. 
Both encoder and decoder are convolutional networks with a symmetric 10-block design; channel widths follow [32, 32, 64, 64, 128, 128, 256, 256, 512, 512], with kernel size 3, stride 2, and padding 1. 
The decoder outputs a 1024-dimensional embedding that is $\ell_2$-normalized to the unit hypersphere. 
We train the VAE to regress from anchor embeddings $c_i$ to target embeddings $x_i^{a}$ using the proxy regression objective in Eq.~\ref{eq:genloss} (AdamW, learning rate $1\times 10^{-4}$, batch size 128) on the ImageNet-1K-VL-Enriched training split.

\subsection{CDF Analysis of Proxy Quality}
\label{cdfanalysis}
Figure~\ref{fig:generative CDF} reports the cumulative distribution function (CDF) of cosine similarity between synthesized proxies and their corresponding real embeddings, evaluated on VGG-S \citep{VGG} and SUN \citep{SUN-D}. 
A curve closer to 1 indicates higher proxy fidelity (i.e., synthesized embeddings better match real-data embeddings). 
As shown in Fig.~\ref{fig:generative CDF}, diffusion-based synthesis achieves consistently higher similarity across the range, indicating more faithful proxy embeddings. 
In contrast, the C-MCR baseline exhibits a noticeable degradation on SUN, consistent with the proxy collapse behavior discussed above.

\section{Downstream Datasets}
\label{Downstream Datasets}
\paragraph{AudioSet} \citep{AudioSet} is a dataset comprising 10-second YouTube videos annotated into 527 classes. It includes a balanced subset with approximately 20,000 videos, a test subset with 18,000 videos, and an unbalanced training subset with around 2 million videos.
For training with image and text as the core modalities, the balanced subset of 16,000 videos is utilized for audio-video and audio-text alignment, respectively. For zero-shot evaluation, as presented in Table \ref{tab:zero-shot x-language}, the test set is employed, with logits computed for each class using textual class names. A total of 16,000 data pairs are used for training.
During text-core modality training and zero-shot evaluation, prompt templates for class names, as detailed in Appendix \ref{Prompt Templates}, are applied. The performance metric reported is the mean Average Precision (mAP).

\paragraph{AudioCaps}\citep{AudioCaps}
is a dataset of audio-visual clips from YouTube with textual descriptions. 
It consists of clips from the AudioSet dataset. 
Following ImageBind, we used the splitting method provided in previous work \cite{oncescu2021audio} to remove clips that overlap with the VGGSound dataset. 
We obtain 48,198 training segments, 418 validation segments, and 796 test segments. 
Only the test set is used for zero-shot evaluation of our model. 
The task is text-to-audio retrieval and is evaluated using recall@K.

\paragraph{ESC-50} \citep{ESC} serves as a benchmark for evaluating the zero-shot capabilities of the learned representations. The dataset focuses on the task of \textsl{Environmental Sound Categorization} (ESC) and comprises 2,000 five-second audio clips distributed across 50 categories. In this study, we perform zero-shot predictions for evaluation purposes. The performance metric employed is the top-1 classification accuracy.

\paragraph{VGG-S}\citep{VGG} comprises approximately 200,000 ten-second video clips annotated with 309 sound categories, covering human actions, sound-producing objects, and human-object interactions. Zero-shot classification and RGB-to-Audio retrieval tasks were conducted using only the audio data from the test set. The evaluation metrics include top-1 accuracy for zero-shot classification and Recall@K for RGB-to-Audio retrieval.

\paragraph{AVE}\citep{AVE} contains 4,143 YouTube videos across 28 event categories and videos in the AVE dataset that are temporally labeled with audiovisual event boundaries.
Evaluations were performed using top-1 accuracy for zero-shot classification and Recall@K for RGB-to-Audio retrieval with the highest accuracy.

\paragraph{Clotho}\citep{clotho}
is an audio dataset with textual descriptions from the \textsl{Freesound} platform. 
It consists of a development set and a test set containing 2893 audio clips and 1045 audio clips respectively, each associated with 5 descriptions. 
We consider the text-to-audio retrieval task and treat each of the 5 associated descriptions as a separate test query, which is then retrieved from the set of audio clips. 
The metric used is recall@K, i.e., a given test query is assumed to be solved correctly if the base fact audio is retrieved in the first K audio clips retrieved.

\paragraph{SUN}\citep{SUN-D} contains about 10,000 RGB-D images.
We follow ImageBind to post process the depth maps in three steps: 1) in-filled depth values, 2) convert them to disparity for scale normalization and 3) limited the minimum and maximum depth to 0.01 and 10 meters respectively.
Training split (about 5,000 data pairs) is used for training models.
Specifically, for text core modality model training, we use prompt templates for the class names as described later in Appendix.\ref{Prompt Templates}.

\paragraph{NYU-D}\citep{NYU} is evaluated using 80\% of the dataset samples. During preprocessing, the minimum and maximum depth values of the depth images were constrained to 0.01 and 10 meters, respectively. Following the approach in ImageBind, we performed a classification and reorganization process that resulted in 10 scene categories. Zero-shot evaluation and RGB-to-Depth retrieval tasks were conducted, with top-1 accuracy and Recall@1 used as the evaluation metrics. For the RGB-to-Depth retrieval task, we utilized prompt templates, as detailed in Appendix \ref{Prompt Templates}.

\paragraph{LLVIP}\citep{LLVIP} is an infrared spectral pedestrian object detection dataset. Consistent with the ImageBind methodology, we extracted all instances of "people" from the images and categorized all other objects as background elements. This preprocessing resulted in a dataset comprising 7622 samples labeled as “background” and 7954 samples labeled as “people” which were subsequently utilized for binary classification tasks. Approximately 5000 infrared-RGB pairs were used for training. Additionally, prompt templates, as detailed in Appendix \ref{Prompt Templates}, were applied in the zero-shot classification task. Since LLVIP is primarily designed for detection tasks and RGB images lack direct text annotations, we employed GPT-4 to generate textual descriptions for each RGB image during the training process for the text core modality.

\paragraph{ImageNet-1K}\cite{imagenet} consists of 1,000 object classes and includes 1.28 million images for training, 5,000 images for validation, and 100,000 images for testing. Building upon this dataset, the ImageNet-1K-VL-Enriched\footnote{The dataset is available at https://huggingface.co/datasets/visual-layer/imagenet-1k-vl-enriched} dataset extends ImageNet-1K by incorporating additional features, including image captions, bounding boxes, and corrected label annotations. For training the diffusion model, caption-image pairs from the training split of the ImageNet-1K-VL-Enriched dataset are utilized.

\section{Additional Theoretical Analysis: Proof of Theorem 1}
\begin{proof}
First, we apply the chain rule to get the gradient of $\mathcal{L}^{\text{osr}}$ with respect to the parameter $\Theta$:
\begin{equation}
\label{ChainRuleGrad1}
\frac{\partial \mathcal{L}^{\text{osr}}}{\partial \Theta} = \frac{\partial x^b_i}{\partial \Theta} \frac{\partial T_{\bar{c}_i}(x^b_i)}{\partial x^b_i} \frac{\partial \mathcal{L}^{\text{osr}}}{\partial T_{\bar{c}_i}(x^b_i)}.
\end{equation}
Similarly, using the chain rule to calculate the gradient of $\mathcal{L}_{align}$ with respect to $\Theta$, we get $\frac{\partial \mathcal{L}^\text{align}}{\partial \Theta} =  \frac{\partial x^b_i}{\partial \Theta}  \frac{\partial \mathcal{L}^{\text{osr}}}{\partial x^b_i}$.
Substituting $\frac{\partial \mathcal{L}^{\text{osr}}}{\partial x^b_i}$ with $\bar{c}_i$ according to Eq.\ref{gradci}, we get
\begin{equation}
\label{ChainRuleGrad2}
\frac{\partial \mathcal{L}^{\text{align}}}{\partial \Theta} = - \frac{\partial x^b_i}{\partial \Theta}\bar{c}_i.
\end{equation}

\label{p1}
Take the inner product $\frac{\partial \mathcal{L}^\text{align}}{\partial \Theta} + \lambda \frac{\partial \mathcal{L}^{\text{osr}}}{\partial \Theta}$ and $\frac{\partial \mathcal{L}^\text{align}}{\partial \Theta}$ to get Eq.\ref{graddot}.
\begin{equation}
\label{graddot}
\begin{split}
&(\frac{\partial \mathcal{L}^\text{align}}{\partial \Theta} + \lambda \frac{\partial \mathcal{L}^{\text{osr}}}{\partial \Theta})^T \frac{\partial \mathcal{L}^\text{align}}{\partial \Theta} = \bar{c}_i^T \frac{\partial x^b_i}{\partial \Theta}^T\frac{\partial x^b_i}{\partial \Theta}\bar{c}_i \\
&+\lambda \frac{\partial \mathcal{L}^{\text{osr}}}{\partial T_{\bar{c}_i}(x^b_i)}^T\frac{\partial T_{\bar{c}_i}(x^b_i)}{\partial x^b_i}^T\frac{\partial x^b_i}{\partial \Theta}^T\frac{\partial x^b_i}{\partial \Theta}\bar{c}_i,
\end{split}
\end{equation}
Notably, we expand $\frac{\partial T_{\bar{c}_i}(x^b_i)}{\partial x^b_i}$ as in Eq.\ref{dot0}, and find that $\bar{c}_i^T\frac{\partial T_{\bar{c}_i}(x^b_i)}{\partial x^b_i}=0$ due to $\bar{c}_i^T\left(I-\frac{\bar{c}_i\bar{c}_i^T}{\|\bar{c}_i\|^2}\right)=0$.
\begin{equation}
\label{dot0}
\begin{split}
\frac{\partial T_{\bar{c}_i}(x^b_i)}{\partial x^b_i} &= \frac{\partial\left(I-\frac{\bar{c}_i\bar{c}_i^T}{\|\bar{c}_i\|^2}\right)x^b_i}{\partial x^b_i}\frac{\partial T_{\bar{c}_i}(x^b_i)}{\partial \left(I-\frac{\bar{c}_i\bar{c}_i^T}{\|\bar{c}_i\|^2}\right)x_i^b} \\
&= \left(I-\frac{\bar{c}_i\bar{c}_i^T}{\|\bar{c}_i\|^2}\right)\frac{\partial T_{\bar{c}_i}(x^b_i)}{\partial \left(I-\frac{\bar{c}_i\bar{c}_i^T}{\|\bar{c}_i\|^2}\right)x_i^b}.
\end{split}
\end{equation}
Next, we introduce Lemma.\ref{lemma1}
\begin{lemma}
\label{lemma1} 
We have $x^TA^TAy \geq -(\kappa(A^TA)-1)^\frac{1}{2}\|Ay\|^2\frac{\|x\|}{\|y\|}$, if $y^Tx = 0$ and $\kappa(\cdot)$ is condistion number.
\end{lemma}
\begin{proof}
    Rather than proving this inequality directly, we turn to the follow the lower bound for the optimization problem.\ref{opt1}.
    \begin{equation}
    \label{opt1}
    \begin{aligned}
        & \min_x \quad && x^TA^TAy, \\
        & \text{subject to} \quad 
            && y^Tx  = 0,\\
            &&& x^Tx = 1,
    \end{aligned}
\end{equation}
To solve optimization problem.\ref{opt1}, we write the Lagrangian function in Eq.\ref{Lagrangian},
\begin{equation}
    \label{Lagrangian}
    L(x,\mu_1, \mu_2) = x^TA^TAy + \mu_1 x^Ty + \mu_2(x^Tx-1).
\end{equation}
We can easily get the Lagrangian dual function Eq.\ref{Langrangian dual} form Eq.\ref{Lagrangian},
\begin{equation}
    \label{Langrangian dual}
    L(\mu_1,\mu_2) = \inf_x L(x,\mu_1,\mu_2) = -\frac{\|(A^TA-\mu_1I)y\|^2}{4\mu_2} -\mu_2,
\end{equation}
Thus, we obtain the unconstrained Lagrangian dual problem.\ref{dual opt}
\begin{equation}
\label{dual opt}
    \max_{\mu_1,\mu_2} L(\mu_1,\mu_2).
\end{equation}
Besides, we have
\begin{equation}
    \label{dual opt1}
    \begin{split}
    \max_{\mu_2} L(\mu_1,\mu_2) &= \|(A^TA-\mu_1I)y\| \\
    &= \left(\mu_1^2y^Ty-2\mu_1y^TA^TAy+y^TA^TAA^TAy\right)^\frac{1}{2}.
    \end{split}
\end{equation}
Thus, we get 
\begin{equation}
    \label{dual opt2}
    \begin{split}
    \max_{\mu_1,\mu_2} L(\mu_1,\mu_2) &= \max_{\mu_1}\max_{\mu_2} L(\mu_1,\mu_2) \\
    &= -\frac{\left(\|A^TAy\|^2\|y\|^2-\|Ay\|^4\right)^\frac{1}{2}}{\|y\|}.
    \end{split}
\end{equation}
Since the maximum eigenvalue of $A^TA$ divided by the minimum eigenvalue of $A^TA$ is less than $2$, we have
\begin{equation}
\label{22}
\begin{split}
    \|A^TAy\|^2\|y\|^2 &\leq \|A\|^2\|Ay\|^2\|y\|^2 \\
    &= \|Ay\|^4\frac{\|A\|^2\|y\|^2}{\|Ay\|^2} \leq \kappa(A^TA)\|Ay\|^4.
\end{split}
\end{equation}
After that, according to (\ref{dual opt2}) and (\ref{22}), we have
\begin{equation}
    \max_{\mu_1,\mu_2}L(\mu_1,\mu_2) \geq -(\kappa(A^TA)-1)^\frac{1}{2}\frac{\|Ay\|^2}{\|y\|}.
\end{equation}
According to the duality principle \citep{boyd2004convex}, we have $\frac{x^T}{\|x\|}A^TAy\geq -(\kappa(A^TA)-1)^\frac{1}{2}\frac{\|Ay\|^2}{\|y\|}$ which completes the proof.
\end{proof}

According to \cite{JMLR:v22:20-410,Arora2018ACA,pmlr-v97-du19c,pmlr-v97-ghorbani19b}, $\kappa\left(\frac{\partial x^b_i}{\partial \Theta}^T\frac{\partial x^b_i}{\partial \Theta}\right)$ will converge and less than 2.
Consider $\frac{\partial T_{\bar{c}_i}(x^b_i)}{\partial x^b_i} \frac{\partial \mathcal{L}^\text{osr}}{\partial T_{\bar{c}_i}(x^b_i)},\bar{c}_i$ and $\frac{\partial x^b_i}{\partial \Theta}$ as $x$, $y$, and $A$ in Lemma.\ref{lemma1}, respectively. According to Lemma.\ref{lemma1}, we can get inequality.\ref{ieq1}
\begin{equation}
\label{ieq1}
\begin{split}
&(\frac{\partial \mathcal{L}^\text{align}}{\partial \Theta} + \lambda \frac{\partial \mathcal{L}^\text{osr}}{\partial \Theta})^T \frac{\partial \mathcal{L}^\text{align}}{\partial \Theta} \geq \\
&\|\frac{\partial x^b_i}{\partial \Theta}\bar{c}_i\|^2(1-\lambda\frac{\|\frac{\partial T_{\bar{c}_i}(x^b_i)}{\partial x^b_i} \frac{\partial \mathcal{L}^\text{osr}}{\partial T_{\bar{c}_i}(x^b_i)}\|}{\|\bar{c}_i\|}).
\end{split}
\end{equation}
Moreover, expanding $\frac{\partial T_{\bar{c}_i}(x^b_i)}{\partial (I-\frac{\bar{c}_i\bar{c}_i^T}{\|\bar{c}_i\|^2})x^b_i}$ in Eq.\ref{dot0} we have 
\begin{equation}
\begin{split}
    &\frac{\partial T_{\bar{c}_i}(x^b_i)}{\partial (I-\frac{\bar{c}_i\bar{c}_i^T}{\|\bar{c}_i\|^2})x^b_i} = \\
    &I - \frac{(I-\frac{\bar{c}_i\bar{c}_i^T}{\|\bar{c}_i\|^2})x^b_i\left((I-\frac{\bar{c}_i\bar{c}_i^T}{\|\bar{c}_i\|^2})x^b_i\right)^T}{\|(I-\frac{\bar{c}_i\bar{c}_i^T}{\|\bar{c}_i\|^2})x^b_i\|^2}.
\end{split}
\end{equation}
Furthermore, we find $(I-\frac{\bar{c}_i\bar{c}_i^T}{\|\bar{c}_i\|^2})$ and $\frac{\partial T_{\bar{c}_i}(x^b_i)}{\partial (I-\frac{\bar{c}_i\bar{c}_i^T}{\|\bar{c}_i\|^2})x^b_i}$ are both projection matrix \citep{horn2012matrix}, and both spectral radius are less than $1$.
Thus, we have
\begin{equation}
\label{spect}
    \|\frac{\partial T_{\bar{c}_i}(x^b_i)}{\partial x^b_i}\frac{\partial \mathcal{L}^\text{osr}}{\partial T_{\bar{c}_i}(x^b_i)}\|\leq\|\frac{\partial \mathcal{L}^\text{osr}}{\partial T_{\bar{c}_i}(x^b_i)}\|,
\end{equation}
due to the spectral radius of $\frac{\partial T_{\bar{c}_i}(x^b_i)}{\partial x^b_i}$ is less than $1$.
Thus, according to inequality .\ref{ieq1} and \ref{spect}, we have
\begin{equation}
    (\frac{\partial \mathcal{L}^\text{align}}{\partial \Theta} + \lambda \frac{\partial \mathcal{L}^\text{osr}}{\partial \Theta})^T \frac{\partial \mathcal{L}^\text{align}}{\partial \Theta} \geq \|\frac{\partial x^b_i}{\partial \Theta}\bar{c}_i\|^2(1-\lambda\frac{\|\frac{\partial \mathcal{L}^\text{osr}}{\partial T_{\bar{c}_i}(x^b_i)}\|}{\|\bar{c}_i\|}),
\end{equation}
which means $(\frac{\partial \mathcal{L}^\text{align}}{\partial \Theta} + \lambda \frac{\partial \mathcal{L}^\text{osr}}{\partial \Theta})^T \frac{\partial \mathcal{L}^\text{align}}{\partial \Theta} \geq  0$ when $\lambda \leq \frac{\|\bar{c}_i\|}{\|\frac{\partial \mathcal{L}^\text{osr}}{\partial T_{\bar{c}_i}(x^b_i)}\|}$ and complete the proof.
\end{proof}

\subsection{InfoNCE vs. Orthogonal-Subspace Regularizer}
\label{Discussion} 
In this section, we provide a preliminary analysis of why directly employing infoNCE to align with synthetic embeddings leads to a decline in the alignment capability of the core modality.
For clarity, we focus on the scenario where the similarity function is defined as cosine similarity, and the temperature parameter $\tau$ is set to 1.
As discussed in Sec. \ref{sec:orthoreg}, $\mathcal{L}^\text{osr}$ facilitates the alignment between the two modalities, whereas $\mathcal{L}^\text{uniform}$ primarily serves as a regularization term.
Thus, the align parts of the two infoNCEs are added together and we get
\begin{equation}
\begin{split}
    \mathcal{L}^\text{align}(\hat{x}^a_i,x^b_i)+{\lambda}\mathcal{L}^\text{align}(c_i,x^b_i) &=-(\hat{x}^a_i)^Tx^b_i- {\lambda}c_i^Tx^b_i \\
    &= -(\hat{x}^a_i+ {\lambda}c_i)^Tx^b_i,
\end{split}
\end{equation}
This implies that when $x^b_i$ is not aligned with either $c_i$ or $\hat{x}^a_i$, but instead with $\text{normalize}(\hat{x}^a_i + {\lambda}c_i)$, the alignment with the core modality is compromised.
In contrast, replacing infoNCE with the OSR modifies the alignment-related loss function to:
\begin{equation}
\label{al_ap}
    \mathcal{L}^\text{align} = -c_i^Tx^b_i-{\lambda}(\hat{x}^a_i)^T\text{normalize}((I-\frac{c_ic_i^T}{\|c_i\|^2})x^b_i).
\end{equation}
It is important to note that, for simplicity of notation, we continue to use $\mathcal{L}^\text{align}$ in Eq.\ref{al_ap}, despite a minor discrepancy between Eq.\ref{al_ap} and Eq.\ref{eqn:align part}.
Letting $\mathcal{L}^\text{align}$ take the derivative of $x^b_i$, we get
\begin{equation}
\label{ap_2}
\begin{split}
    \frac{\partial \mathcal{L}^\text{align}}{\partial x^b_i} = -c_i - {\lambda}(I-\frac{x^b_i(x^b_i)^T}{\|x^b_i\|^2})(I-\frac{c_ic_i^T}{\|c_i\|^2}) \\ \cdot (I - \frac{(I-\frac{\bar{c}_i\bar{c}_i^T}{\|\bar{c}_i\|^2})x^b_i\left((I-\frac{\bar{c}_i\bar{c}_i^T}{\|\bar{c}_i\|^2})x^b_i\right)^T}{\|(I-\frac{\bar{c}_i\bar{c}_i^T}{\|\bar{c}_i\|^2})x^b_i\|^2})\hat{x}^a_i.
\end{split}
\end{equation}
Since $x^b_i$ is confined to the hypersphere, $(I-\frac{x^b_i(x^b_i)^T}{\|x^b_i\|^2})$ appears in Eq.\ref{ap_2}.
Taking an inner product of $c_i$ and $\frac{\partial \mathcal{L}^\text{align}}{\partial x^b_i}$, we have
\begin{equation}
\begin{split}
    &c_i^T\frac{\partial \mathcal{L}^\text{align}}{\partial x^b_i} = -c_i^Tc_i - {\lambda}c_i^T(I-\frac{x^b_i(x^b_i)^T}{\|x^b_i\|^2})(I-\frac{c_ic_i^T}{\|c_i\|^2}) \\
    &\cdot (I - \frac{(I-\frac{\bar{c}_i\bar{c}_i^T}{\|\bar{c}_i\|^2})x^b_i\left((I-\frac{\bar{c}_i\bar{c}_i^T}{\|\bar{c}_i\|^2})x^b_i\right)^T}{\|(I-\frac{\bar{c}_i\bar{c}_i^T}{\|\bar{c}_i\|^2})x^b_i\|^2})\hat{x}^a_i.
\end{split}
\end{equation}
It is easy to find that 
\begin{equation}
    c_i^T(I-\frac{x^b_i(x^b_i)^T}{\|x^b_i\|^2}) = c_i^T-\frac{c_i^Tx^b_i}{\|x^b_i\|^2}(x^b_i)^T,
\end{equation}
where 
\begin{equation}
    c_i^T(I-\frac{c_ic_i^T}{\|c_i\|^2}) = 0,
\end{equation}
and 
\begin{equation}
\begin{split}
    &\frac{c_i^Tx^b_i}{\|x^b_i\|^2}\left((x^b_i)^T(I-\frac{c_ic_i^T}{\|c_i\|^2})\right) \\
    &\cdot (I - \frac{(I-\frac{\bar{c}_i\bar{c}_i^T}{\|\bar{c}_i\|^2})x^b_i\left((I-\frac{\bar{c}_i\bar{c}_i^T}{\|\bar{c}_i\|^2})x^b_i\right)^T}{\|(I-\frac{\bar{c}_i\bar{c}_i^T}{\|\bar{c}_i\|^2})x^b_i\|^2}) = 0.
\end{split}
\end{equation}
Thus, we have
\begin{equation}
    c_i^T\frac{\partial \mathcal{L}^\text{align}}{\partial x^b_i} = -c_i^Tc_i \leq 0,
\end{equation}
which means the similarity between $x^b_i$ and $c_i$ is on increasing.

\end{document}